\documentclass[sigconf]{acmart}
\AtBeginDocument{%
  \providecommand\BibTeX{{%
    \normalfont B\kern-0.5em{\scshape i\kern-0.25em b}\kern-0.8em\TeX}}}

\usepackage[utf8]{inputenc} 
\usepackage[T1]{fontenc}    
\usepackage{hyperref}       
\usepackage{url}            
\usepackage{booktabs}       
\usepackage{amsfonts}       
\usepackage{nicefrac}       
\usepackage{microtype}      
\usepackage{xcolor}         
\usepackage{amsmath}
\usepackage{algorithm}
\usepackage{algorithmic}
\usepackage{graphicx}
\usepackage{textcomp}
\usepackage{multirow}
\usepackage{subfigure}
\usepackage{pifont}
\usepackage{paralist}
\usepackage{mathtools}
\usepackage{amsthm}
\usepackage{float}
\usepackage{wrapfig}
\usepackage{mathrsfs}

\usepackage{newfloat}
\usepackage{listings}

\theoremstyle{plain}
\newtheorem{theorem}{Theorem}[section]

\newtheorem{lemma}[theorem]{Lemma}

\theoremstyle{definition}

\newtheorem{assumption}[theorem]{Assumption}
\theoremstyle{remark}

\begin{document}

\title{Is Aggregation the Only Choice? Federated Learning via Layer-wise Model Recombination}


\copyrightyear{2024}
\acmYear{2024}
\setcopyright{acmlicensed}
\acmConference[KDD '24] {Proceedings of the 30th ACM SIGKDD Conference on Knowledge Discovery and Data Mining }{August 25--29, 2024}{Barcelona, Spain.}
\acmBooktitle{Proceedings of the 30th ACM SIGKDD Conference on Knowledge Discovery and Data Mining (KDD '24), August 25--29, 2024, Barcelona, Spain}
\acmISBN{979-8-4007-0490-1/24/08}
\acmDOI{10.1145/3637528.3671722}


\begin{abstract}

Although Federated Learning (FL) enables global model training across clients without compromising their raw data, due to the unevenly distributed data among clients, existing Federated Averaging (FedAvg)-based methods suffer from the problem of low inference performance. Specifically, different data distributions among clients lead to various optimization directions of local models. Aggregating local models usually results in a low-generalized global model, which performs worse on most of the clients. To address the above issue, inspired by the observation from a geometric perspective that a well-generalized solution is located in a flat area rather than a sharp area, we propose a novel and heuristic FL paradigm named FedMR (Federated Model Recombination). The goal of FedMR is to guide the recombined models to be trained towards a flat area. Unlike conventional FedAvg-based methods, in FedMR, the cloud server recombines collected local models by shuffling each layer of them to generate multiple recombined models for local training on clients rather than an aggregated global model. Since the area of the flat area is larger than the sharp area, when local models are located in different areas, recombined models have a higher probability of locating in a flat area. When all recombined models are located in the same flat area, they are optimized towards the same direction. We theoretically analyze the convergence of model recombination. Experimental results show that, compared with state-of-the-art FL methods, FedMR can significantly improve the inference accuracy without exposing the privacy of each client.
\end{abstract}

\ccsdesc[300]{Computing methodologies~Distributed artificial intelligence}

\keywords{Federated Learning, Model Recombination, Non-IID, Generalization}



\author{Ming Hu}
\email{hu.ming.work@gmail.com}
\orcid{0000-0002-5058-4660}
\affiliation{%
  \institution{Singapore Management University}
  \city{Singapore}
  \country{Singapore}
}

\author{Zhihao Yue}
\email{51215902034@stu.ecnu.edu.cn}
\affiliation{%
  \institution{East China Normal University}
  \city{Shanghai}
  \country{China}
}

\author{Xiaofei Xie}
\email{xfxie@smu.edu.sg}
\affiliation{%
  \institution{Singapore Management University}
  \city{Singapore}
  \country{Singapore}
}

\author{Cheng Chen}
\email{chchen@sei.ecnu.edu.cn}
\affiliation{%
  \institution{East China Normal University}
  \city{Shanghai}
  \country{China}
}

\author{Yihao Huang}
\email{huangyihao22@gmail.com}
\affiliation{%
  \institution{Nanyang Technological University}
  \city{Singapore}
  \country{Singapore}
}

\author{Xian Wei}
\email{xwei@sei.ecnu.edu.cn}
\affiliation{%
  \institution{East China Normal University}
  \city{Shanghai}
  \country{China}
}

\author{Xiang Lian}
\email{xlian@kent.edu}
\affiliation{%
  \institution{Kent State University}
  \city{Ohio}
  \country{USA}
}

\author{Yang Liu}
\email{yangliu@ntu.edu.sg}
\affiliation{%
  \institution{Nanyang Technological University}
  \city{Singapore}
  \country{Singapore}
}

\author{Mingsong Chen}
\authornote{Corresponding author.}
\email{mschen@sei.ecnu.edu.cn}
\affiliation{%
  \institution{East China Normal University}
  \city{Shanghai}
  \country{China}
}

\renewcommand{\shortauthors}{Ming Hu et al.}
\maketitle

\section{Introduction}
Federated Learning (FL) \cite{fedavg,wang2023fedcda,li2023towards} has been widely acknowledged as a promising means to design large-scale distributed Artificial Intelligence (AI) applications, e.g., Artificial Intelligence of Things (AIoT) systems~\cite{tcad_zhang_2021,hu2023aiotml,wang2024feddse}, healthcare systems~\cite{kdd_qian_2021,adnan2022federated}, and recommender systems~\cite{tan2020federated,wen2023survey}.
Unlike conventional Deep Learning (DL) methods, the cloud-client architecture based FL supports the collaborative training of a global DL model among clients without compromising their raw data~\cite{nips_naman_2021}.
In each FL training round, the cloud server first dispatches the global model to its selected clients for local training and then gathers the corresponding gradients of trained models from clients for aggregation.
In this way, clients can train a global model without sharing data.

Although FL enables effective collaborative training among multiple clients while protecting data privacy, existing FL methods suffer from the problem of ``weight divergence''~\cite{kairouz2021advances}.
Especially when the data on the clients are non-IID (Identically and Independently Distributed)~\cite{infocom_hao_2020,iclr_2021_durmus}, the optimal directions of local models on clients and the aggregated global model on the cloud server are significantly inconsistent, resulting in serious inference performance degradation of the global model. 
To improve FL performance in non-IID scenarios, various FL methods have been studied, e.g., client grouping-based methods~\cite{axiv_ming_2021}, global control variable-based methods~\cite{icml_scaffold,aaai_yutao_2021,fedprox}, Knowledge Distillation (KD)-based methods\cite{nips_tao_2020,icml_zhuangdi_2021,wang2023dafkd}, and mutation-based methods~\cite{hu2024fedmut}. 
The basic ideas of these solutions are to guide the local training on clients~\cite{icml_scaffold,aaai_yutao_2021} or adjust parameter weights for model aggregation~\cite{nips_tao_2020,icml_zhuangdi_2021}.

Although these methods are promising in alleviating the impact of data heterogeneity, most of them adopt the well-known Federated Averaging (FedAvg)-based paradigm, which may potentially reduce generalization performance.
This is mainly because FedAvg paradigm only aggregates the parameters of collected local models and initializes local training by clients with the same global models.
Specifically, since the data distribution among clients is different, the optimal directions of the local models are diverse.
On the one hand, although the aggregation operation can achieve knowledge sharing among multiple local models, it can still neglect the specific knowledge learned by local models, which seriously limits the inference performance of the global model.
On the other hand, since FedAvg only uses the same global model for local training, FL training inevitably results in notorious {\it stuck-at-local-search} problems during local training.
As a result, the global model based on simple statistical averaging cannot accurately reflect both individual efforts and the potential of local models in the search for optimal global models.
Therefore, \textit{how to overcome the shortcomings of the FedAvg-based paradigm and improve the performance of FL in non-IID scenarios is an important challenge.}

Some recent research on model training indicates
that, from the perspective of the loss landscapes of DL models, optimal solutions with well generalization performance often lie in flat valleys, while the inferior ones are always located in sharp ravines~\cite{hochreiter1997flat,kwon2021asam}. 
Inspired by the above observation, to collaboratively train a well-generalized model, FL needs to guide the local training towards a more flat area.
Since the direction of gradient descent is stochastic, compared to using the same global model, using multiple global models for local training has a greater probability that the existing model can optimize to a flatter area.
Since flat areas are usually larger than sharp areas, intuitively, the exchange of the corresponding parameters among multiple models rather than aggregation can allow them to be displaced in the solution space. 
When a model is stuck in a sharp area, the parameter exchange may make it escape from the sharp area.
With continuous training and parameter exchange, when multiple models are located in the same flat area, these models will optimize in the same direction, that is, the center of the flat area.

Inspired by the above intuition, this paper proposes a novel FL paradigm called FedMR 
 ({\bf Fed}erated {\bf M}odel {\bf R}ecombination), which can effectively help the training of local models  escape from sharp area. 
Unlike FedAvg that aggregates all the collected
local models in each FL training round, FedMR randomly shuffles the parameters of different local models within the same layers, and recombines them to form new local models.
%
In this way, FedMR can derive diversified models that can effectively escape local optimal solutions for the local training of clients. 
%
The main contributions of this paper can be summarized as follows:
\begin{compactitem}
    \item
We propose a novel FL paradigm named FedMR, which contains a newly layer-wise model recombination method to replace the traditional FedAvg-based model aggregation with the aim of improving FL inference performance.


\item 
We introduce a
two-stage training scheme for FedMR, which combines the merits of both model recombination and aggregation to accelerate the overall FL training process. 

\item We theoretically prove the convergence of FedMR in convex scenarios and conduct empirical experiments to validate the convergence of FedMR in non-convex scenarios.

\item 
 We conduct extensive experiments on various well-known models and datasets to show both the effectiveness and compatibility of FedMR. 



\end{compactitem}



\section{Related Work}\label{sec:relatedwork}
To address the  problem of uneven data distributions, 
exiting solutions  
can be mainly
classified into four categories, i.e., client grouping-based methods, global control variable-based methods, knowledge distillation-based methods, and mutation-based methods.
The {\it device grouping-based methods}
group and select clients for aggregation based on the data similarity between clients. For example, FedCluster~\cite{bigdata_cheng_2020} divides clients into multiple clusters and performs multiple cycles of  meta-update to boost the overall FL convergence. 
Based on either sample size or model similarity, CluSamp~\cite{icml_yann_2021} groups clients to achieve a better client representativity and a reduced variance of client stochastic aggregation parameters in FL.
By modifying the  penalty terms of loss functions during  FL 
training, the {\it global control variable-based methods}
can be used to smooth the FL convergence process. 
For example, FedProx~\cite{fedprox} regularizes local loss functions with the squared distance between local models and the global model to stabilize the model convergence. Similarly, SCAFFOLD~\cite{icml_scaffold} uses global control variables to correct the ``client-drift'' problem in the local training process. 
{\it Knowledge Distillation (KD)-based methods} adopt soft targets generated by the ``teacher model'' to guide the training of ``student models''.
For example, by leveraging a proxy dataset, Zhu et al.~\cite{icml_zhuangdi_2021} proposed a data-free knowledge distillation method named FedGen to address the heterogeneous FL problem using a built-in generator. With ensemble distillation, FedDF~\cite{nips_tao_2020} accelerates the FL training by training 
 the global model through unlabeled data on the outputs of local models. 
{\it Mutation-based methods} attempt to mutate the global model to generate multiple mutated intermediate models for local training.
For example, FedMut~\cite{hu2024fedmut} utilizes the gradients to mutate the global model and dispatches the mutated models for local training.


Although the above methods can optimize FL performance from different perspectives, since coarse-grained model aggregation is performed, the inference capabilities of local models are still strongly restricted.
Furthermore, most of them cannot avoid non-negligible communication and computation overheads or the risk of data privacy exposure.
In addition, many FL methods have been proposed to address device heterogeneity problems.
To effectively train on devices with different hardware resources, some methods~\cite{alam2022fedrolex,wang2023fedlego,jia2023adaptivefl,ilhan2023scalefl} utilize heterogeneous models for local training.
To avoid stragglers caused by uneven computing capability or uncertainty~\cite{hu2020quantitative}, existing methods~\cite{hu2023gitfl,xia2024cabafl}  attempt to perform a wise client scheduling to achieve asynchronous FL training. 
Note that this paper only focuses on the data heterogeneity problem.

To the best of our knowledge,  FedMR is the first attempt using model recombination rather than aggregation for FL. 
Since FedMR considers the specific characteristics and efforts of local models, it can further mitigate the weight divergence problem, thus achieving better inference performance than state-of-the-art   FL methods. 

\section{Motivation}\label{sec:motivation}
\subsection{Intuition}
\textbf{Comparison between FedAvg and Independent Training}.
Figure~\ref{fig:motivation} illustrates
the FL training processes on the same loss landscape using FedAvg and Independent training (Indep), respectively, where the server of Indep just shuffles the received local models and then randomly dispatches them to clients without aggregation. 
In each subfigure, the local optima 
are located within the areas surrounded by solid red lines. Note that since  the upper surrounded area   
is flatter than the lower surrounded area in the loss landscape, 
the solutions within it will exhibit 
better generalization.


\begin{figure}[htp] 
	\begin{center} 
    \includegraphics[width=0.3\textwidth]{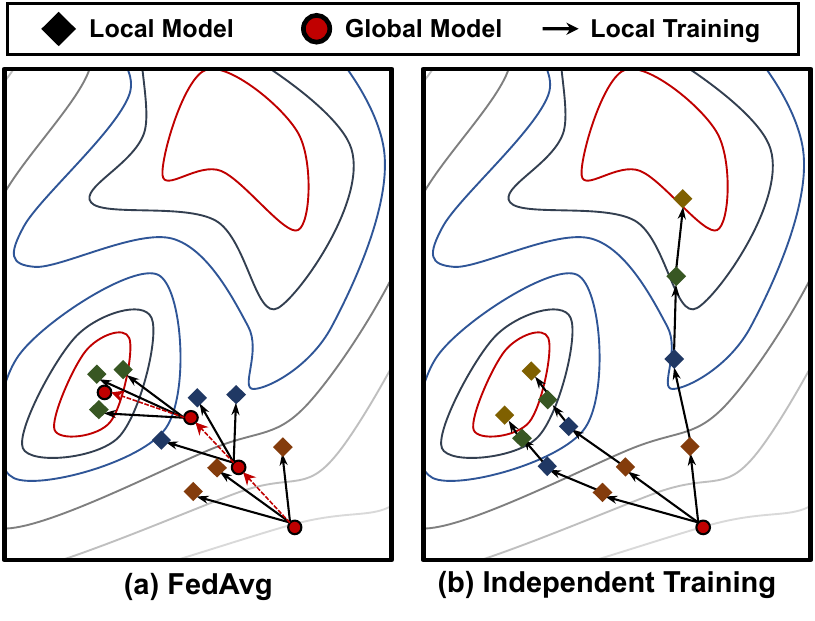}
  \caption{Training processes on the same loss landscape.}
		\label{fig:motivation} 
	\end{center}
\end{figure}

As shown in Figure~\ref{fig:motivation}(a),
along with the training process,
the aggregated global models denoted by 
red circles gradually move toward the lower
sharp area with inferior solutions, though 
the optimizations  of some local models  head toward the upper surrounded area
with better solutions. 
Such biased training is mainly because the local training starts from the same global model in each FL round. 
As an alternative, due to the lack of aggregation operation, the local models of Indep may converge in different directions as shown in 
Figure~\ref{fig:motivation}(b). 
In this case, even if 
some local training in Idep
achieves a better solution than 
the one  obtained by FedAvg, 
due to the diversified optimization directions of local models, such an important finding
can  be eclipsed by the 
results of other local models. 
Clearly,  there is a lack of 
mechanisms for Indep that can 
guide the overall training  
toward such superior solutions.

\begin{figure}[htp] 
	\begin{center} 
		\includegraphics[width=0.42\textwidth]{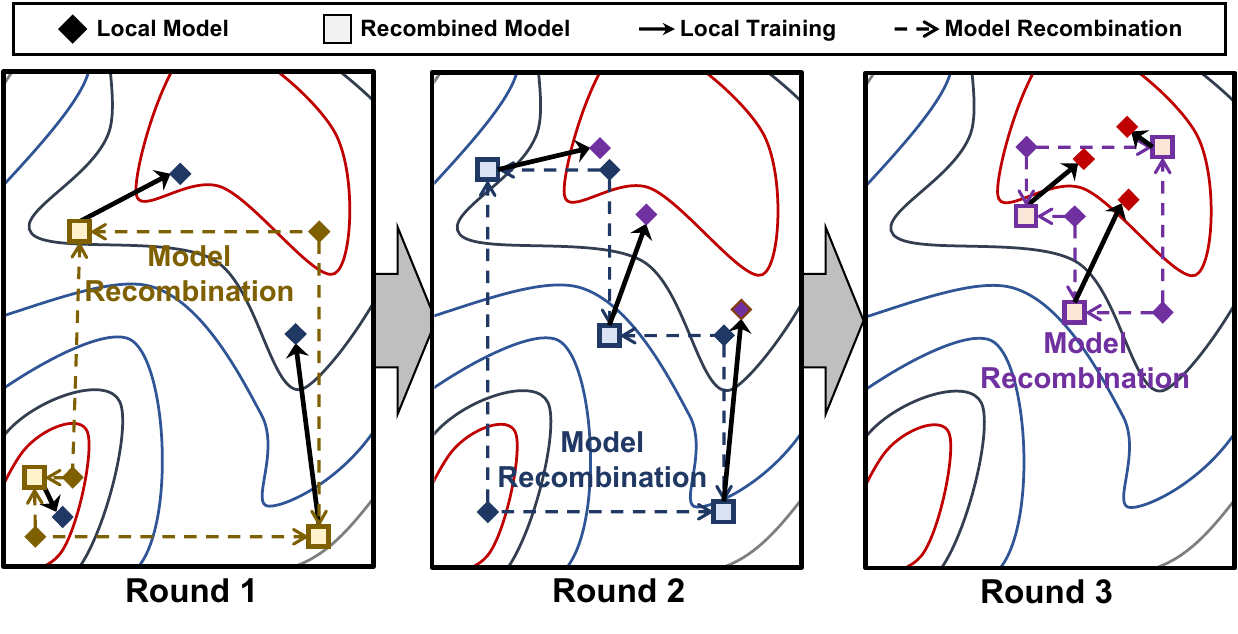}
  \caption{An example of model recombination.}
		\label{fig:motivation_b} 
	\end{center}
\end{figure}

\textbf{Intuition of Model Recombination.}
Based on the  Indep training results shown in
Figure~\ref{fig:motivation}(b), 
Figure~\ref{fig:motivation_b} illustrates
the intuition of our model recombination method, where the 
FL training starts from the three local models (denoted by  
yellow diamonds in round 1) obtained in figure~\ref{fig:motivation}(b). At the beginning of round 1, 
two of the three local models are located in the sharp ravine.
In other words, without model recombination, 
 the training of such two local models may get stuck in the lower surrounded
 area. However,  due to the weight adjustment by shuffling the layers among local models,   we can find that 
the three recombined models (denoted by yellow squares) are sparsely scattered in the loss landscape, which enables the local training
escape from local optima. 
According to \cite{hardt2016train,wu2020adversarial}, 
a small perturbation of the model weights can make it easier for local training to jump out of sharp ravines rather than flat valleys.
In other words, the recombined 
models are more likely to converge toward
 flat valleys along the local training. 
For example, in the end of  round 3, we can find that all three local models are located in the upper surrounded area, where their aggregated model has better generalization performance
than the one achieved in Figure~\ref{fig:motivation}(a).

\begin{figure}[h] 
	\begin{center} 
	\includegraphics[width=0.45\textwidth]{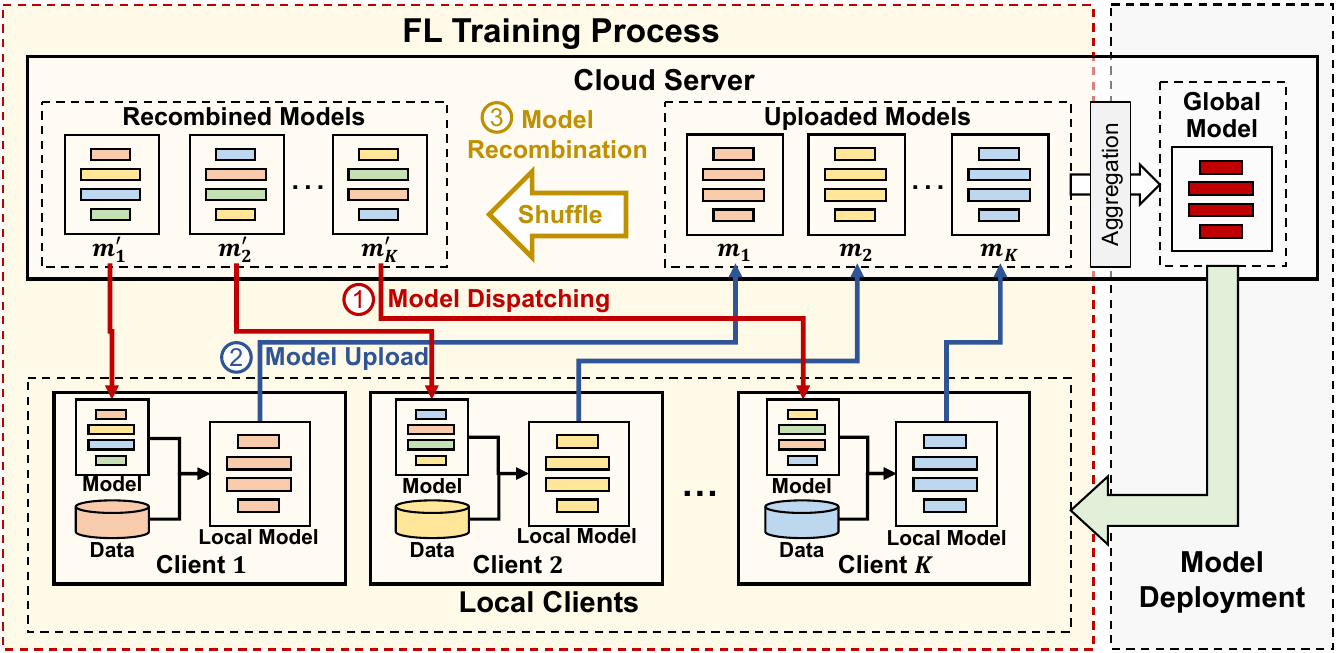}
  \caption{Our FedMR approach}
		\label{fig:framework} 
	\end{center}
\end{figure} 

\section{FedMR Approach}

\subsection{Overview of FedMR}

Figure~\ref{fig:framework} presents the framework and workflow of FedMR, which consists of two process, i.e., FL training process and the model deployment process.
In FL training process, unlike FedAvg-based methods that use the same global model for local training, FedMR adopts multiple homogeneous intermediate models for local training, where each client model is dispatched one model.
Specifically, each training round involves three specific steps as follows: 
\begin{itemize}
\setlength\itemsep{-0.1em}
    \item 
    \textbf{Step 1 (Model Dispatching):} The cloud server dispatches $K$ intermediate models to 
    $K$ selected clients, respectively, according to their indices, where $K$ denotes the number of activated clients participating in each FL training round. Note that in FedMR different clients will receive different models for the local training.
    
    
    \item \textbf{Step 2 (Model Upload):} 
    Once its local training is finished, a client needs to upload the parameters of its newly updated
    local model to the cloud server.
    
    \item \textbf{Step 3 (Model Recombination):} 
    The cloud server decomposes received   
    local models into multiple layers individually
    in the same manner and conducts the random shuffling of the same layers among different local models. Then, by concatenating layers
    from different sources in order,  a new local model can be 
    reconstructed. Note that any decomposed layer of the uploaded model will eventually be used by one and only one of the 
    recombined models. 
\end{itemize}

In the model deployment process, the cloud server aggregates all the intermediate models to generate a global model and deploys the global model to all the local clients for specific tasks.

\subsection{Implementation of FedMR}

Algorithm~\ref{alg:fedmr} details  
the implementation of FedMR.
 Line~\ref{line:init} initializes the model list $L_m$, which includes $K$ initial models. Lines~2-10 performs  
$rnd$ rounds of  FedMR training.
In each round,  Line~3  selects $K$ random clients to participate in the model training and creates a client list  $L_r$.
Lines 4-7 conduct the local training on clients in parallel, 
where Line 5 applies the local model $L_m[i]$ on client $L_r[i]$ for local training by using the
function {\it ClientUpdate},
and Line 6 achieves a new local model after the local training. 
After the cloud server receives 
all the  $K$ local models,  Line~8 uses the 
function {\it ModelRcombine}  to recombine local models and 
generate $K$ new local models, 
which are saved in $L_m$ as shown in Line 9.
Finally, Lines 11-12 will report   
an optimal global model that is generated based on $L_m$.
Note that the cloud server will dispatch the global model to all
the clients for the  purpose of inference
rather than local training. 
The following parts  will detail the 
key components of FedMR.
Since FedMR cannot adapt to the secure aggregation mechanism~\cite{bonawitz2017practical}, to further protect privacy, we present an extended secure recombination mechanism in Appendix~\ref{sec:sec_mr}, which enables the exchange of partial layers among clients before local training or model uploading to ensure that the cloud server cannot directly obtain the gradients of each local model.

\begin{algorithm}[t]
\caption{Implementation of  FedMR}
\label{alg:fedmr}
\begin{flushleft}
\textbf{Input}:
i) $rnd$, \# of training rounds; 
ii) $S_{c}$, the set of clients; 
iii) $K$, \# of clients participating in each FL round.\\
\textbf{Output}: 
 $w^{glb}$, the parameters of trained global model.\\
\textbf{FedMR}($rnd$,$S_{dev}$,$K$)
\end{flushleft}
\begin{algorithmic}[1] 
\STATE $L_m\leftarrow [w^1_1, w^2_1,...,w^{K}_1]\ \ \ $ // initialize model list  \;\label{line:init}
\FOR{$r$ = 1, ..., $rnd$}\label{line:trainStart}
\STATE $L_r\leftarrow$ Random select $K$ clients from $S_{c}$\;\label{line:clientSel}
\\ /*parallel for block*/
\FOR{$i$ = 1, ..., $K$}   
\STATE $v_{r+1}^i\leftarrow ${\it ClientUpdate}$(L_m[i],L_r[i])$\;\label{line:clientUpdate}
\STATE $L_m[i]\leftarrow v_{r+1}^i$\;
\ENDFOR
\STATE $[w_{r+1}^1,w_{r+1}^2,..., w_{r+1}^K]\leftarrow ${\it ModelRcombine}$(L_m)$\label{line:modelRecomb}\;
\STATE $L_m\leftarrow [w_{r+1}^1,w_{r+1}^2,..., w_{r+1}^K]$\;
\ENDFOR\label{line:trainEnd}
\STATE $w^{glb}\leftarrow \frac{1}{K}\sum_{i = 1}^{K} w_{rnd+1}^{i}$ \label{line:modelAggr}\;
\STATE \textbf{return} $w^{glb}$\;
\end{algorithmic}
\end{algorithm}

\subsubsection{Local Model Training}
Unlike conventional FL methods that 
conduct local  
training on clients starting from the same 
aggregated  model, in each training round FedMR uses
different recombined models (i.e., $K$ models in the 
model list $L_m$) for the local training purpose. 
Note that, in the whole training phase,
FedMR only uses  $K$ ($K\leq |S_c|$) models, since there are only $K$ devices activated in each training round. 
Let $w^c_r$ be the parameters of some
model that is dispatched 
to the $c^{th}$ client in the $r^{th}$ training round. In the $r^{th}$ training round, we dispatch 
the $i^{th}$ model in $L_m$  to its corresponding client 
using
$w^{L_r[i]}_{r} = L_m[i]$. 
Based on the  recombined model, FedMR conducts the local training on client $L_r[i]$ as follows:
\begin{equation}
\begin{split}
v^{L_r[i]}_{r+1} &=w^{L_r[i]}_{r} - \eta \nabla f_{L_r[i]}(w^{L_r[i]}_{r}),\\
s.t., \ f_{L_r[i]}(w^{L_r[i]}_{r}) &= \frac{1}{|D_{L_r[i]}|} \sum_{j = 1}^{|D_{L_r[i]}|} \ell (w^{L_r[i]}_{r};x_j;y_j),
\end{split}
\nonumber
\end{equation}
where $v^{L_r[i]}_{r}$ indicates parameters of the trained local model, $D_{L_r[i]}$ denotes
the dataset of client $L_r[i]$, 
$\eta$ is the  learning rate, $\ell(\cdot)$ is the loss function, $x_j$ is the $j^{th}$ sample in $D_{L_r[i]}$, and $y_j$ is the label of $x_j$.
Once the local training is finished, 
 the  client needs to upload the parameters of 
 its trained 
  local model  to the cloud server by updating 
  $L_m$ using 
$L_m[i] = v^{L_r[i]}_{r+1}$. 
Similar to traditional FL methods, 
in each training round, FedMR needs to transmit 
the parameters of $2K$ models between the cloud server and its selected clients. 

\begin{figure}[t] 
	\begin{center} 
	\includegraphics[width=0.42\textwidth]{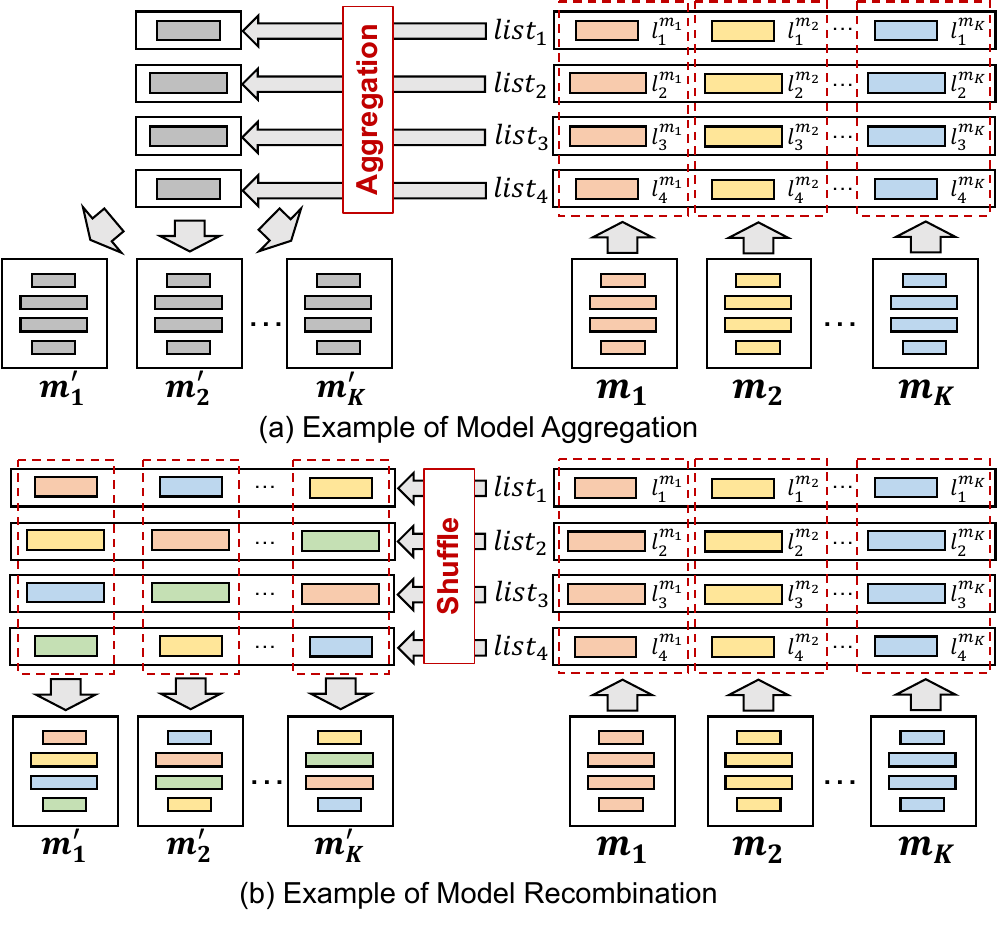}
  \caption{Example of model aggregation and recombination}
		\label{fig:mr_example} 
	\end{center}
\end{figure}

\subsubsection{Model Recombination}
Typically, a DL model consists of multiple layers, e.g., convolutional layers, pooling layers, and Fully Connected (FC) layers. 
To simplify the description of our model recombination method, we do not explicitly present the layer types here. 
Let $w_x=\{l_1^x,l_2^x,...,l_n^x\}$ be
the parameters of  model $x$, where $l_i^x$ ($i\in [n]$) denotes the  parameters of the $i^{th}$ layer of model 
$x$. 
%


In each FL round,  FedMR needs to perform model recombination based on $L_m$ to obtain new models for local training. 
Figure~\ref{fig:mr_example} presents an example of model aggregation and model recombination.
When clients receive all the trained local models (i.e., $m_1$, $m_2$, ..., $m_K$), the cloud server needs to decouple the layers of these models individually.
Note that since all the local models are homogeneous, the corresponding layers of the local models have the same structure.
For example, the model $m_1$ can be decomposed into 
four layers. Assuming that the local models are with an architecture of $w$, to enable recombination, FedMR then constructs $n$ lists, where the $k^{th}$ ($k\in [n]$) list contains all the $k^{th}$ layers of the models in $L_m$.
As an example shown in Figure~\ref{fig:mr_example}(b), FedMR constructs four lists (i.e., $list_1$-$list_4$)
for the $K$ models (i.e., $m_1$-$m_K$), where each list consists of $K$ elements (i.e.,  $K$ layers with the same index). 
After shuffling each list, FedMR generates $|L_m|$ recombined models based on shuffled results.  
For example, the top three layers of the recombined model $m^\prime_1$ come from the models $m_1$, $m_2$ and $m_K$, respectively.
For comparison, Figure~\ref{fig:mr_example}(a) presents an example of model aggregation, which aggregates the layers of each list to generate an aggregated model.
The aggregation-based methods dispatch the aggregated model to $K$ clients, while FedMR dispatches a different recombined model to each client.



\subsection{Two-Stage Training Scheme for FedMR} 

Although  FedMR enables finer
 FL training, when  starting from blank models,  
 FedMR converges more slowly than 
traditional FL methods at the beginning. This is mainly because, due to the low matching degree between layers in the recombined models,
the model recombination operation in this stage 
requires more local training time to re-construct the new dependencies between layers. To accelerate the overall convergence, we 
propose a  two-stage training scheme for FedMR, consisting of the  {\it aggregation-based pre-training stage} and  {\it model recombination stage}.
 In the first stage, we 
train the local models coarsely  using the 
 FedAvg-based aggregation, which can quickly form a pre-trained 
 global model. In the second stage, from the
 pre-trained models, FedMR dispatches recombined models to clients
 for local training. Due to the synergy 
  of both FL paradigms, the overall FedMR training time can be reduced.

\subsection{Convergence Analysis}


Based on the assumptions posed on the loss functions of local clients in FedAvg~\cite{convergence}, this subsection performs the convergence analysis for FedMR. 

\begin{assumption}\label{asm1}
For $i \in\{ 1, 2, \cdots, K\}$, $ f_i $ is L-smooth satisfying $|| \nabla f_i(x) - \nabla f_i(y) || \leq L ||x - y||$.
\end{assumption}
\begin{assumption}\label{asm2}
For $i \in\{ 1, 2, \cdots, K\}$, $ f_i $ is $\mu$-strongly convex satisfying $f(x)\geq f(y) + (x-y)^T\nabla f(y) + \frac{\mu}{2} ||x-y||^2$, where $\mu \ge 0$.

\end{assumption}
\begin{assumption}\label{asm3}
The variance of stochastic gradients is upper bounded by  $\theta^2$, and the expectation of squared norm of stochastic gradients is upper bounded by  $G^2$, i.e.,  $\mathbb{E}||\nabla f_k (w;\xi) - \nabla f_k (w) ||^2 \leq \theta^2$, $\mathbb{E}||\nabla f_k (w;\xi) ||^2 \leq G^2$, where $\xi$ is a data batch of the $k^{th}$ client in the $t^{th}$ FL round.
\end{assumption}

Based on the implementation of  function {\it ModelRecombine($\cdot$)}, 
we  derive the following two lemmas for the model recombination:
\begin{lemma}\label{key_lamma1}
Assume that in FedMR
there are $K$ clients participating in every FL training round. 
Let $\{v^1_r,v^2_r,..,v^K_r\}$ and $\{w^1_r,w^2_r,..,w^K_r\}$ be
the set  of trained local model weights and the set of recombined model weights 
generated in the $(r-1)^{th}$ round, respectively. Assume $x$ is a vector with the same size as that of  $v_{r}^k$. We have
\begin{equation}\label{eq:mr1}
\scriptsize
\sum_{k=1}^K v_{r}^k= \sum_{k=1}^K w_{r}^k,
\  and \ 
\sum_{k=1}^K||v_{r}^k-x||^2=\sum_{k=1}^K||w_{r}^k-x||^2.
\nonumber
\end{equation}
\end{lemma}
We prove Theorem \ref{thm1} based on   Lemmas~\ref{key_lamma1}.  Please refer to Appendix \ref{sec:proof} for the proof. 
Note that different from FedAvg, Lemmas~\ref{key_lamma1} is the key lemma for the proof of FedMR.

\newtheorem{thm}{\bf Theorem}
\begin{thm}\label{thm1}
(Convergence of FedMR)
Let Assumption \ref{asm1},  \ref{asm2}, and \ref{asm3} hold.
Assume that $E$ is the number of SGD iterations conducted within one FL  round, model recombination is conducted at the end of each FL round, and 
the whole training terminates after $n$ FL rounds. 
Let $T=n\times  E$ be the total number of SGD iterations conducted so far, and $\eta_k=\frac{2} {\mu (T + \gamma)}$ be the learning rate. We can have 
\begin{equation}
\begin{split}
\mathbb{E}[f(\overline{w}_T)] -f^\star \leq \frac{L}{2\mu(T+\gamma)}[\frac{4B}{\mu} + \frac{\mu(\gamma+1)}{2}\Delta_1]
     \nonumber
\end{split},
\end{equation}
where
$
    B = 10 L \Gamma + 4(E - 1)^2 G^2, \overline{w}_T=\sum_{k=1}^K{w^k_T}.
     \nonumber
$
\end{thm} 


Theorem \ref{thm1} indicates that the difference between the current loss $f(\overline{w}_T)$ and the optimal loss $f^\star$ is inversely related to $t$. From Theorem \ref{thm1}, we can find that the convergence rate of 
 FedMR
is similar to that of FedAvg, which has been analyzed in \cite{convergence}.

\section{Experimental Results}
\subsection{Experimental Settings}\label{sec:settings}
To evaluate the effectiveness of FedMR, we implemented FedMR
on top of a  cloud-based
architecture. Since it is impractical to allow 
all the clients to get involved in the training processes simultaneously, 
we assumed that there are only 10\% of clients participating in the local training 
in each FL round. To enable fair comparison, all the investigated 
FL methods including FedMR set
SGD optimizer with a learning rate of 0.01 and a momentum of 0.9.
For each client,  we set the batch size of local training to 50,  and performed 
five epochs for each local training. All the experimental results were obtained 
from an Ubuntu workstation with Intel i9 CPU, 32GB memory, and NVIDIA  RTX 3080 GPU.

{\bf Baseline Method Settings.}
We compared the test accuracy of FedMR with seven baseline  methods, i.e.,  FedAvg~\cite{fedavg}, FedProx~\cite{fedprox}, FedGen~\cite{icml_zhuangdi_2021}, CluSamp~\cite{icml_yann_2021}, FedExP~\cite{jhunjhunwala2023fedexp}, FedASAM~\cite{fedasam}, and FedMut~\cite{hu2024fedmut}. 
Here,  FedAvg is the most classical FL method, while the other five methods are state-of-the-art (SOTA) representations of the four kinds of FL optimization methods introduced in the related work section.
Specifically, FedProx, FedExP, and FedASAM are global control variable-based methods, FedGen is a KD-based approach, CluSamp is a device grouping-based method, and FedMut is a mutation-based method.
For FedProx, we used a hyper-parameter $\mu$  to control the weight of its proximal term, where the best values of $\mu$  for CIFAR-10, CIFAR-100, and FEMNISTvare 0.01, 0.001, and 0.1, respectively.
For FedGen, we adopted the same server settings in \cite{icml_zhuangdi_2021}.
For CluSamp, the clients were clustered based on the model gradient similarity described in \cite{icml_yann_2021}.



{\bf Dataset  Settings.}
We investigated the performance of our approach on 
three well-known datasets, i.e., CIFAR-10,  CIFAR-100 \cite{data}, 
and FMNIST~\cite{leaf}. 
We adopted
the Dirichlet distribution~\cite{measuring} to control the heterogeneity of 
client data for both CIFAR-10  and CIFAR-100.
Here, the notation $Dir(\alpha)$ indicates a different Dirichlet distribution
controlled by $\alpha$, where a smaller $\alpha$ means higher data heterogeneity
of clients. Note that, different from datasets CIFAR-10 and CIFAR-100, 
the raw data of FEMNIST are naturally non-IID distributed. Since FEMNIST takes 
various kinds of data heterogeneity
into account, we did not apply the 
Dirichlet distribution on FEMNIST. 
For both CIFAR-10 and CIFAR-100, we assumed that there are 100 clients in total participating
in FL. For FEMNIST,  we only considered one non-IID scenario 
involving 180 clients, where each client hosts more than 100 local
data samples.

{\bf Model  Settings.}
To demonstrate the pervasiveness of our approach, we 
developed different FedMR implementations
based on 
three different models (i.e., CNN, ResNet-20, VGG-16). 
Here, we obtained the CNN model from \cite{fedavg}, which 
consists of 
two convolutional layers and two FC layers. 
When conducting 
FedMR based on the CNN model, we directly applied
the model recombination for local training on 
it without pre-training a global model, since CNN here  
only has four layers. 
We obtained both ResNet-20 and VGG-16 models from Torchvision \cite{models}. ·  
When performing
FedMR based on ResNet-20 and VGG-16, 
due to the deep structure of both models, 
we adopted the two-stage training scheme, where the first stage 
lasts for 100 rounds 
to obtain a pre-trained global model.


\subsection{Validation for Intuitions}

\textbf{Independent Training.}
Based on the settings presented in Section~\ref{sec:settings}, we conducted the experiments to  evaluate
the effectiveness of each local model in Indep.
The FL training is based on the ResNet-20 model and dataset CIFAR-10, where we set  $\alpha= 0.5$ for non-IID scenarios. 
Figures~\ref{fig:obs}
compares Indep with  FedAvg from the perspectives of 
both test loss and inference accuracy. 
Due to the space limitation, for Indep here we only  
present the results of its four random local models (denoted by  Model-1, Model-2, Model-3, and Model-4).
To enable  a fair comparison with FedAvg, 
although there is no aggregated global model in  Indep,
we  considered 
 the aggregated model of all its local models for
 each FL round, whose results are indicated by the notion ``IndepAggr''. 
From Figure~\ref{fig:obs}, we can find 
all the local models in Indep can achieve higher accuracy and lower loss than those of FedAvg, though their loss and accuracy curves fluctuate more sharply.
Moreover, IndepAggr exhibits much worse performance than the other references. 
This is mainly because, according to the definition of Indep,  each local model needs to traverse multiple 
clients along with the FL training processes, where  the 
optimization directions of client models differ in the corresponding loss landscape. 

\begin{figure}[htp]
\centering
	\subfigure[Test Loss]{
		\centering
		\includegraphics[width=0.22\textwidth]{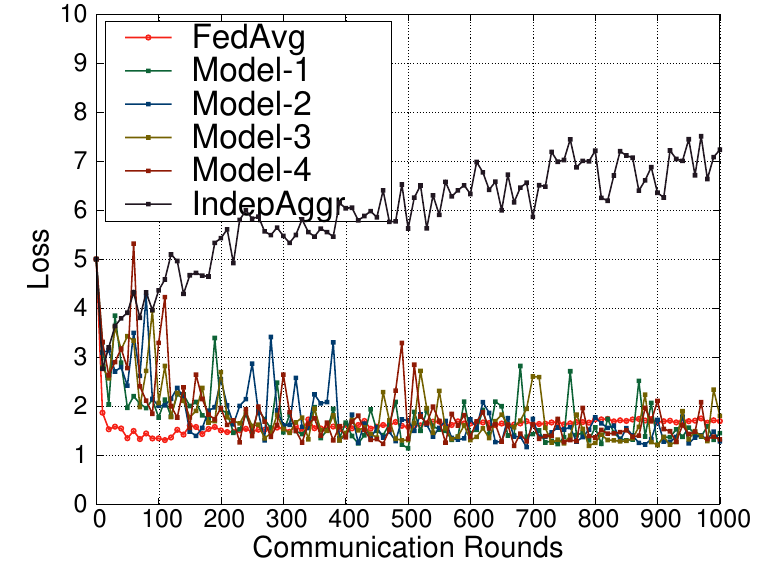}
	}
    \subfigure[Accuracy]{
		\centering
		\includegraphics[width=0.22\textwidth]{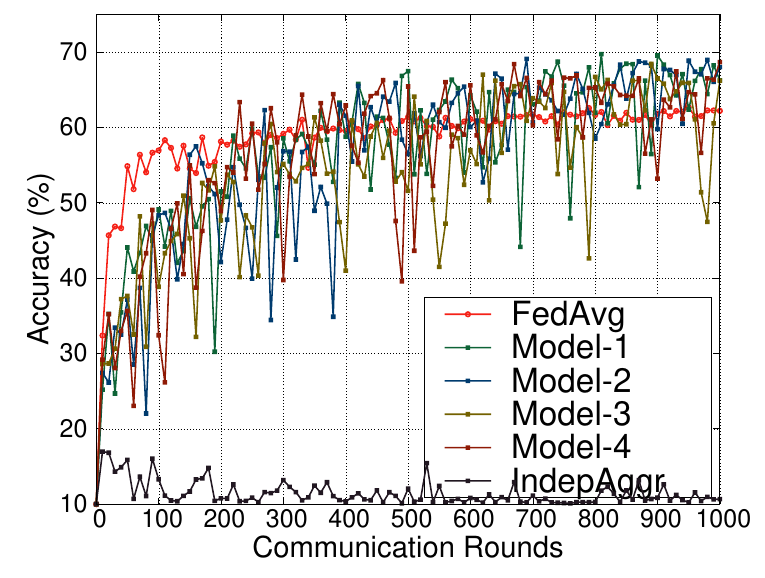}
	}
 \caption{FedAvg vs. Indep.} 
 \label{fig:obs}
\end{figure}

\textbf{Model Recombination.}
To validate the intuition  about the 
 impacts of model recombination as presented in Section~\ref{sec:motivation}, we conducted 
three experiments on CIFAR-10 dataset using ResNet-20 model.
Our goal is to figure out the following three  questions:
i) by using model recombination, can all the models  eventually  have the  same optimization direction;
 ii) compared with FedAvg, can the global model of FedMR  eventually converge into a more flat solution; and
 iii) can the global model of FedMR eventually converge to a more generalized solution?

\begin{figure}[htp]
\begin{center} 		\includegraphics[width=0.22\textwidth]{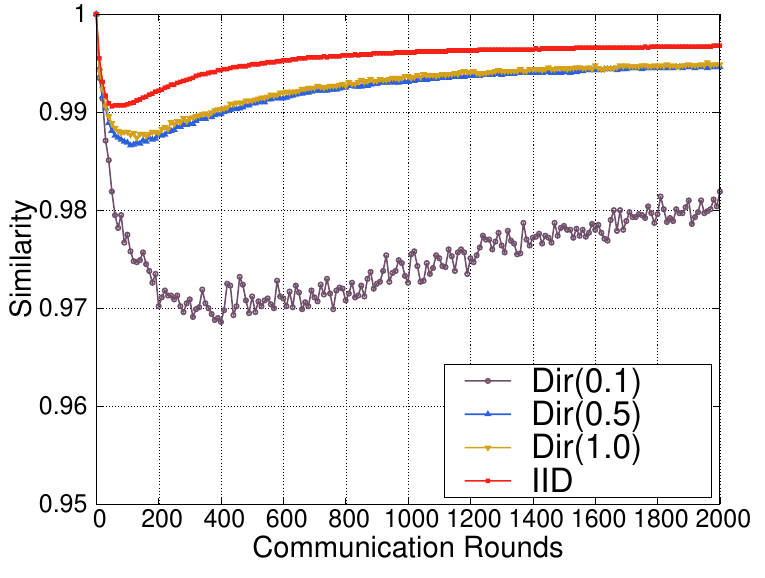}
	\caption{Cosine similarity of local models in FedMR.}
	\label{fig:intuition_sim}
 \end{center}
\end{figure}
 
Figure~\ref{fig:intuition_sim} presents the average cosine similarity between all the intermediate models, taking four different client data distributions into account.
We can observe that the model similarity decreases first and gradually increases in all the investigated IID and non-IID scenarios. 
Due to the nature of 
Stochastic Gradient Descent (SGD) mechanism and the data heterogeneity among clients, 
all local models are optimized toward different directions at the beginning of training. However, as the training progresses, most local models will be located in the same flat valleys, leading to similar optimization directions for local models. 
These results are consistent with our intuition as shown in 
Figure~\ref{fig:motivation_b}.

\begin{figure}[h]
\centering
\subfigure[FedAvg w/ $\alpha=0.1$]{
		\centering
		\includegraphics[width=0.2\textwidth]{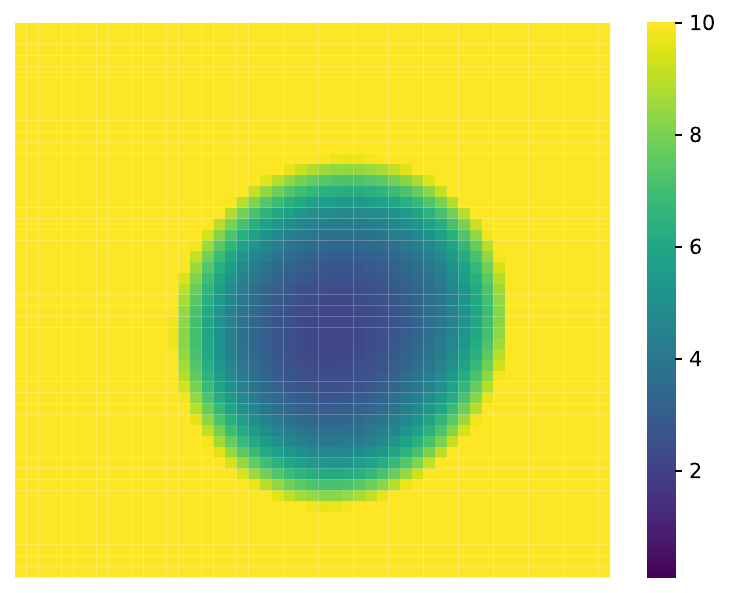}
	}\hspace{-0.1in}
    \subfigure[FedMR w/  $\alpha=0.1$]{
		\centering
		\includegraphics[width=0.2\textwidth]{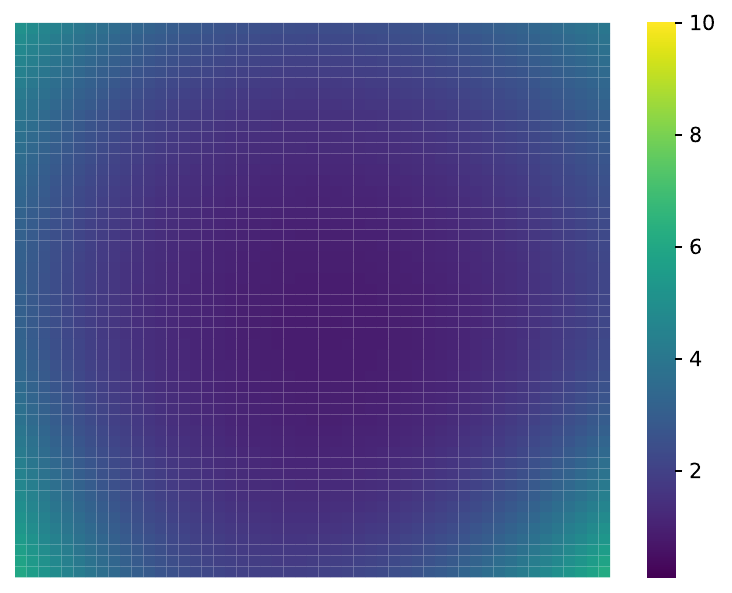}
	}\hspace{-0.1in}
	\subfigure[FedAvg w/  $\alpha=1.0$]{
		\centering
		\includegraphics[width=0.2\textwidth]{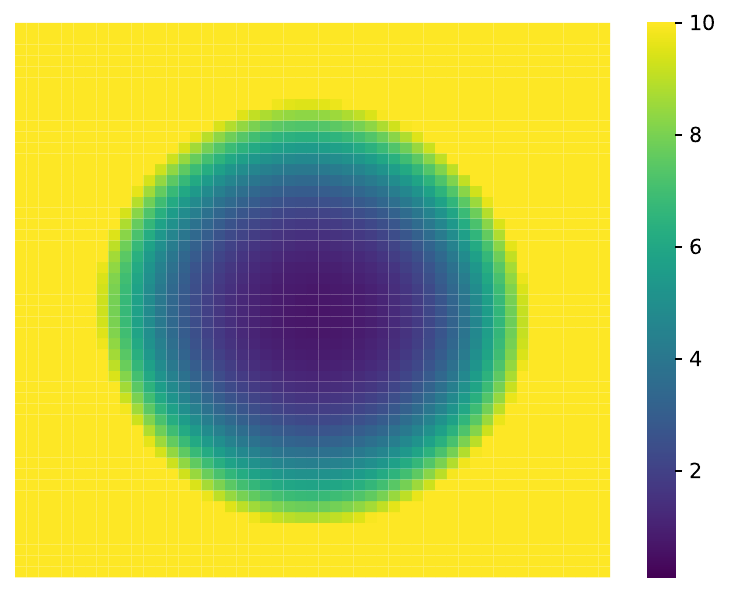}
	}\hspace{-0.1in}
    \subfigure[FedMR w/  $\alpha=1.0$]{
		\centering
		\includegraphics[width=0.2\textwidth]{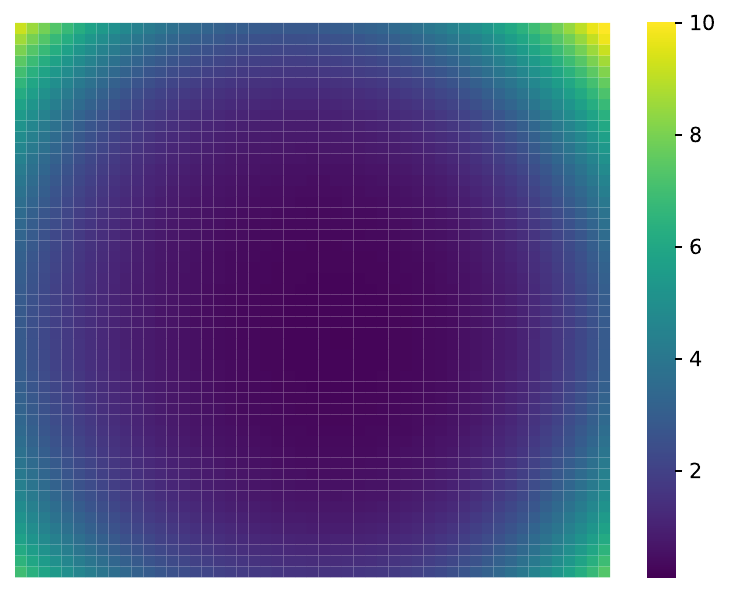}
	}
 \caption{Comparison of loss landscapes with different FL and client data settings.} 
 \label{fig:val_landscape}
\end{figure}

\begin{figure}[t]
	\centering
	\subfigure[$\alpha=0.1$]{
		\centering
		\includegraphics[width=0.2\textwidth]{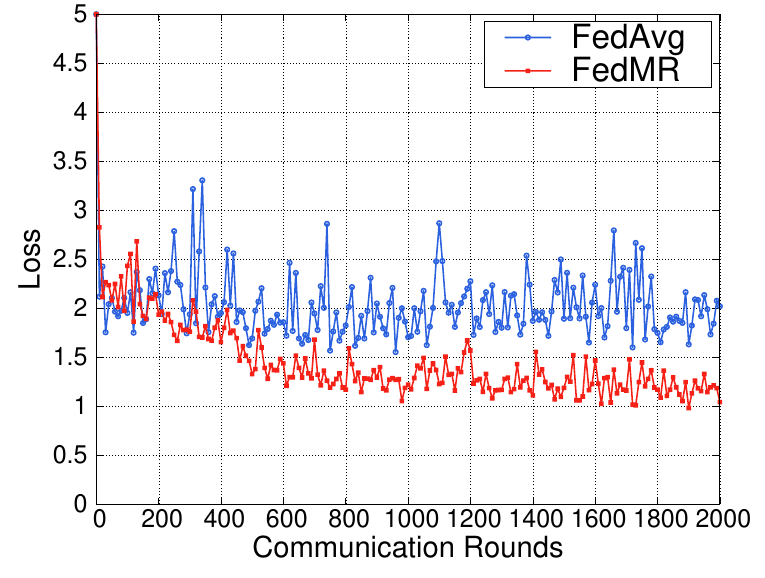}
		\label{fig:loss-d0.1}
	}\hspace{-0.1in}
	\subfigure[$\alpha=0.5$]{
		\centering
		\includegraphics[width=0.2\textwidth]{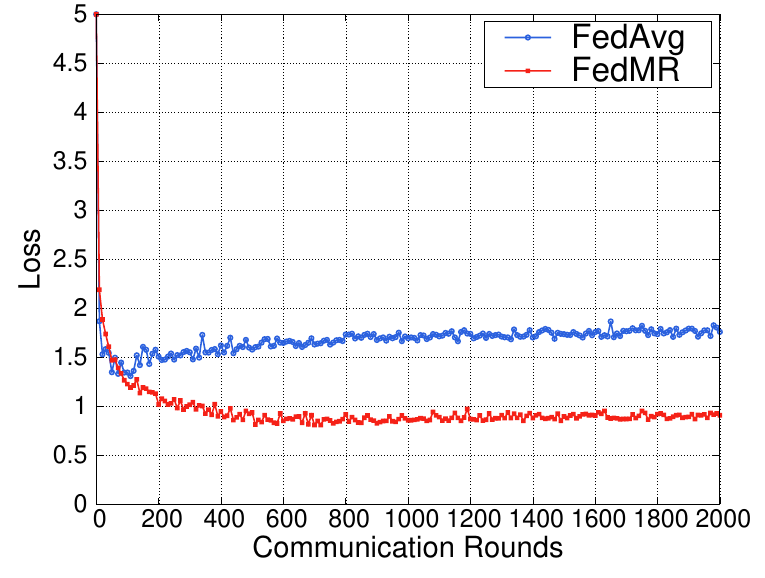}
		\label{fig:loss-d0.5}
	}\hspace{-0.1in}
	\subfigure[$\alpha=1.0$]{
		\centering
		\includegraphics[width=0.2\textwidth]{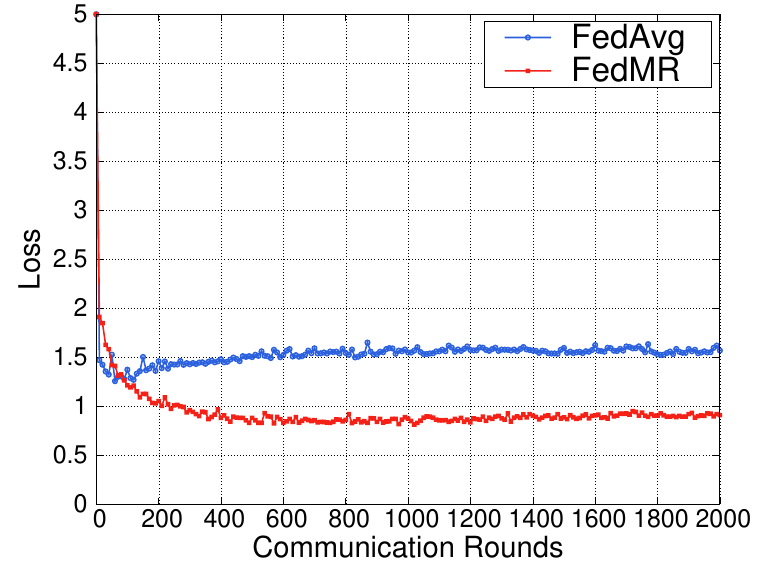}
		\label{fig:loss-d1.0}
	}\hspace{-0.1in}
	\subfigure[IID]{
		\centering
		\includegraphics[width=0.2\textwidth]{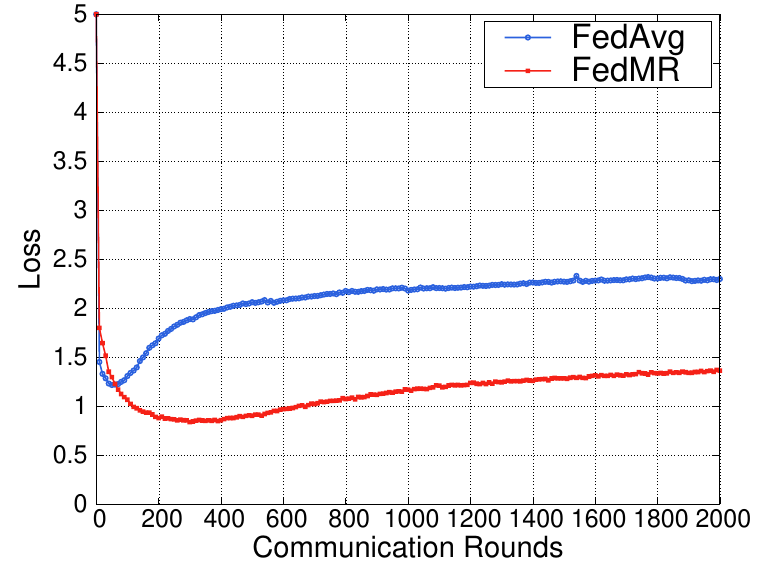}
		\label{fig:loss-iid}
	}
	\caption{Comparison of test losses of FedAvg and FedMR with different client data settings.}
	\label{fig:intuition_loss}
\end{figure}
		
		
Figure~\ref{fig:val_landscape} compares 
the loss landscapes of final global models obtained by FedAvg and FedMR with different client data settings, respectively. 
We can find
that, compared with FedMR counterparts,
the global models trained by FedAvg are located in sharper solutions, indicating the generalization superiority of final global models achieved by FedMR.


Figure~\ref{fig:intuition_loss} compares
the test losses for 
the global models of FedAvg and FedMR (without using two-stage training) within different IID and non-IID scenarios.
Note that here, the global models of FedMR are only for the purpose of fair comparison rather than local model initialization.  Due to the superiority in generalization, we can observe that the models trained by FedMR outperform those by FedAvg for all four cases.

\begin{table*}[t]

\centering
\caption{Test accuracy comparison for both non-IID and  IID scenarios using three DL models}
\label{tab:acc}
\setlength{\tabcolsep}{2 pt}
\begin{tabular}{|c|c|c|c|c|c|c|c|c|c|c|}
\hline
\multirow{2}{*}{Model} & \multirow{2}{*}{Datas.} & Heter. & \multicolumn{8}{c|}{Test Accuracy (\%)} \\
\cline{4-11}
 & & Set. & FedAvg & FedProx & FedGen & CluSamp & FedExP & FedASAM & FedMut & {\bf FedMR} \\
\hline
\hline
\multirow{9}{*}{CNN} & \multirow{4}{*}{Cifar-10} & $0.1$& $50.66\pm 1.18$& $50.61\pm 1.21$ & $ 50.27\pm 2.03 $ & $50.24\pm 1.04$  & $53.20\pm 1.33$ & $48.83\pm 1.86$ & \underline{$52.88\pm 0.59$} & ${\bf 54.63\pm 0.55}$\\
                      &  & $0.5$ & $54.42\pm 0.76$ & $54.28\pm 1.29$ & $53.66\pm 0.68$  & $55.07\pm 0.93$ & $55.12\pm 0.57$ & $52.02\pm 0.43$ & \underline{$57.29\pm 0.61$} & ${\bf 59.81\pm 0.60}$\\
                      &  & $1.0$ & $57.03\pm 0.62$ & $56.44\pm 0.68$ & $56.19\pm 0.54$  & $56.16\pm 0.56$ & $57.12\pm 0.34$ & $55.08\pm 0.61$ & \underline{$58.88\pm 0.47$} & ${\bf 60.77\pm 0.48}$\\
                      &  & $IID$ & $57.21\pm 0.27$ & $57.00\pm 0.11$ & $57.35\pm 0.19$ & $58.13\pm 0.35$ & $57.05\pm 0.15$ & $54.60\pm 0.11$ & \underline{$58.81\pm 0.21$} & ${\bf 63.74\pm 0.18}$\\
\cline{2-11}
& \multirow{4}{*}{Cifar-100} & $0.1$ & $29.12\pm 0.52$ & $29.44\pm 0.73$ & $26.14\pm 0.92$  & $28.51\pm 1.14$  & $30.43\pm 0.39$ & $28.95\pm 0.48$ & \underline{$31.23\pm 0.35$} & ${\bf 34.47\pm 0.54}$\\
                     &  & $0.5$     & $32.41\pm 0.77$ & $32.48\pm 0.81$ & $29.19\pm 0.67$  & $32.63\pm 0.60$ & $33.12\pm 0.51$ & $32.06\pm 0.98$ & \underline{$34.46\pm 0.69$} & ${\bf 37.41\pm 0.30}$\\
                     &  & $1.0$      & $32.66\pm 0.51$ & $33.10\pm 0.41$ & $29.94\pm 0.51$  & $32.65\pm 0.48$ & $33.32\pm 0.38$ & $32.45\pm 0.61$ & \underline{$35.12\pm 0.42$} & ${\bf 39.15\pm 0.30}$\\
                     &  & $IID$ & $32.75\pm 0.20$ & $32.57\pm 0.21$ & $30.95\pm 0.32$ & $32.77\pm 0.11$ & $32.48\pm 0.14$ & $32.35\pm 0.29$ & \underline{$34.49\pm 0.23$} & ${\bf 40.64\pm 0.17}$\\
\cline{2-11}
& \multirow{1}{*}{FEMNIST} & $-$ & $82.97\pm 0.37$ & $83.15\pm 0.41$ & $82.35 \pm 0.40$  & $82.31\pm 0.32$ & $83.28\pm 0.29$ & $83.49\pm 0.27$ & \underline{$83.62\pm 0.33$} & ${\bf 83.76\pm 0.24}$ \\
\hline

\hline
\multirow{9}{*}{ResNet-20} & \multirow{4}{*}{Cifar-10} & $0.1$& $42.14\pm 3.91$& $43.25\pm 3.18$ & $44.19\pm 2.42$  & $41.64\pm 2.04$ & $45.18\pm 2.43$ & $45.22\pm 4.06$ & \underline{$50.75\pm 1.85$} & ${\bf 59.20\pm 1.22}$\\
                      &  & $0.5$ & $58.70\pm 0.86$ & $59.33\pm 0.77$ & $60.64\pm 0.83$  & $58.74\pm 0.82$ & $59.74\pm 0.92$ & $63.49\pm 1.10$ & \underline{$63.34\pm 0.70$} & ${\bf 72.41\pm 0.17}$\\
                      &  & $1.0$ & $64.33\pm 0.25$ & $64.75\pm 0.33$ & $64.41\pm 0.29$  & $63.42\pm 0.45$ & $64.48\pm 0.31$ & $67.62\pm 0.41$ & \underline{$68.09\pm 0.25$} & ${\bf 75.16\pm 0.31}$\\
                      &  & $IID$ & $65.72\pm 0.22$ & $65.95\pm 0.23$ & $66.31\pm 0.23$ & $65.36\pm 0.18$ & $65.58\pm 0.24$ & $69.65\pm 0.10$ & \underline{$69.77\pm 0.19$} & ${\bf 77.48\pm 0.10}$\\
\cline{2-11}
& \multirow{4}{*}{Cifar-100} & $0.1$ & $34.22\pm 1.01$ & $34.52\pm 0.68$ & $35.76\pm 1.11$  & $33.23\pm 0.78$ & $36.09\pm 0.71$ & $37.31\pm 0.66$ & \underline{$38.79\pm 0.55$} & ${\bf 44.61\pm 1.48}$\\
                     &  & $0.5$     & $42.16\pm 0.43$ & $41.37\pm 0.49$  & $45.03\pm 0.96$ & $41.54\pm 0.52$ & $41.96\pm 0.56$ & $44.29\pm 0.52$ & \underline{$46.55\pm 0.55$} & ${\bf 54.26\pm 0.45}$\\
                     &  & $1.0$     & $43.32\pm 0.38$ & $43.00\pm 0.45$ & $46.60\pm 0.39$  & $43.63\pm 0.39$ & $43.68\pm 0.23$ & $46.74\pm 0.31$ & \underline{$48.41\pm 0.25$} & ${\bf 55.72\pm 0.36}$\\
                     &  & $IID$ & $45.14\pm 0.28$ & $45.40\pm 0.27$ & $48.33\pm 0.25$ & $44.76\pm 0.24$ & $45.04\pm 0.24$ & $48.59\pm 0.20$ & \underline{$48.65\pm 0.17$} & ${\bf 59.24\pm 0.32}$\\
\cline{2-11}
& \multirow{1}{*}{FEMNIST} & $-$ & $79.09\pm 0.54$ & $78.89\pm 0.50$ & $79.56\pm 0.34$  & $78.75\pm 0.27$  & $79.15\pm 0.34$ & \underline{$81.20\pm 0.41$} & $78.98\pm 0.45$ & ${\bf 81.27\pm 0.31}$ \\
\hline

\hline
\multirow{9}{*}{VGG-16} & \multirow{4}{*}{Cifar-10} & $0.1$& $62.28\pm 5.72$& $63.19\pm 5.15$ & $65.97\pm 3.82$  & $62.00\pm 3.19$ & $64.78\pm 5.69$ & $65.52\pm 4.96$ & \underline{$69.15\pm 2.16$} & ${\bf 74.49\pm 0.92}$\\
    &  & $0.5$ & $78.82\pm 0.21$ & $78.49\pm 0.26$ & $78.98\pm 0.11$  & $78.09\pm 0.48$ & $79.30\pm 0.43$ & $79.12\pm 0.25$ & \underline{$80.07\pm 0.19$} & ${\bf 84.24\pm 0.37}$\\
                      &  & $1.0$ & $79.53\pm 0.34$ & $79.52\pm 0.30$  & $80.08\pm 0.24$ & $79.67\pm 0.45$ & $79.76\pm 0.20$ & $80.08\pm 0.32$ & \underline{$80.85\pm 0.99$} & ${\bf 85.12\pm 0.13}$\\
                      &  & $IID$ & $79.96\pm 0.05$ & $79.79\pm 0.07$ & $80.13\pm 0.05$  & $79.66\pm 0.06$ & $79.89\pm 0.05$ & $80.66\pm 0.09$ & \underline{$82.20\pm 0.05$} & ${\bf 85.66\pm 0.15}$\\
\cline{2-11}
& \multirow{4}{*}{Cifar-100} & $0.1$ & $47.29\pm 0.96$ & $48.02\pm 0.68$ & $49.11\pm 1.60$  & $47.38\pm 1.47$ & $49.36\pm 0.55$ & $48.78\pm 1.03$ & \underline{$51.30\pm 1.00$} & ${\bf 55.33\pm 0.72}$\\
                     &  & $0.5$     & $55.60\pm 0.55$ & $55.45\pm 0.70$ & $56.29\pm 0.84$  & $54.45\pm 0.58$ & $56.40\pm 0.47$ & $56.73\pm 0.41$ & \underline{$58.02\pm 0.31$} & ${\bf 65.07\pm 0.25}$\\
                     &  & $1.0$      & $56.05\pm 0.45$ & $55.75\pm 0.36$ & $57.96\pm 0.32$  & $55.70\pm 0.37$ & $56.69\pm 0.25$ & $57.46\pm 0.25$ & \underline{$58.53\pm 0.31$} & ${\bf 65.66\pm 0.15}$\\
                     &  & $IID$ & $57.22\pm 0.28$ & $56.65\pm 0.23$  & $58.47\pm 0.16$ & $57.33\pm 0.17$ & $57.55\pm 0.16$ & $57.96\pm 0.11$ & \underline{$58.62\pm 0.08$} & ${\bf 66.33\pm 0.10}$\\
\cline{2-11}
& \multirow{1}{*}{FEMNIST} & $-$ & $83.96\pm 0.43$ & $84.27\pm 0.32$ & $84.39\pm0.28$ & $83.64\pm 0.27$  & $83.99\pm 0.32$ & \underline{$84.59\pm 0.21$} & $83.97\pm 0.57$ &  ${\bf 85.36\pm 0.21}$\\
\hline
\end{tabular}
\end{table*}

\begin{figure*}[h]
	\centering
	\subfigure[$\alpha=0.1$]{
		\centering
		\includegraphics[width=0.22\textwidth]{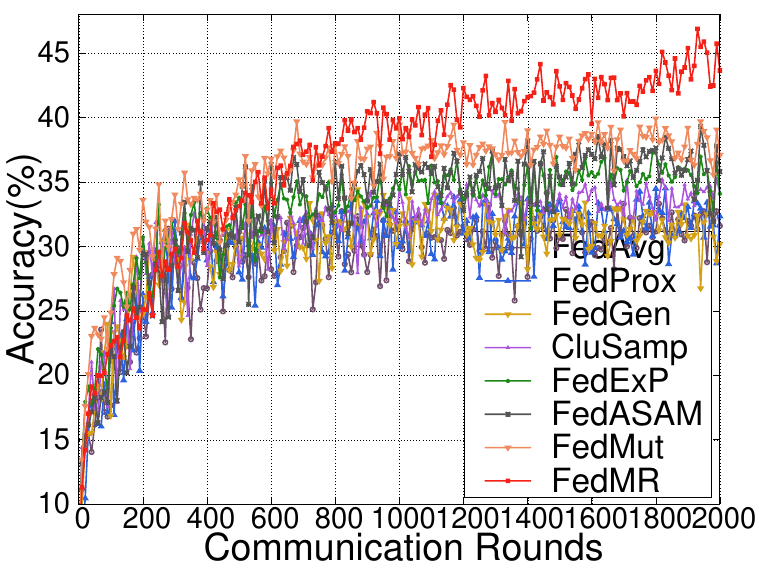}
		\label{fig:cifar-10-0.1}
	}\hspace{-0.1in}
	\subfigure[$\alpha=0.5$]{
		\centering
		\includegraphics[width=0.22\textwidth]{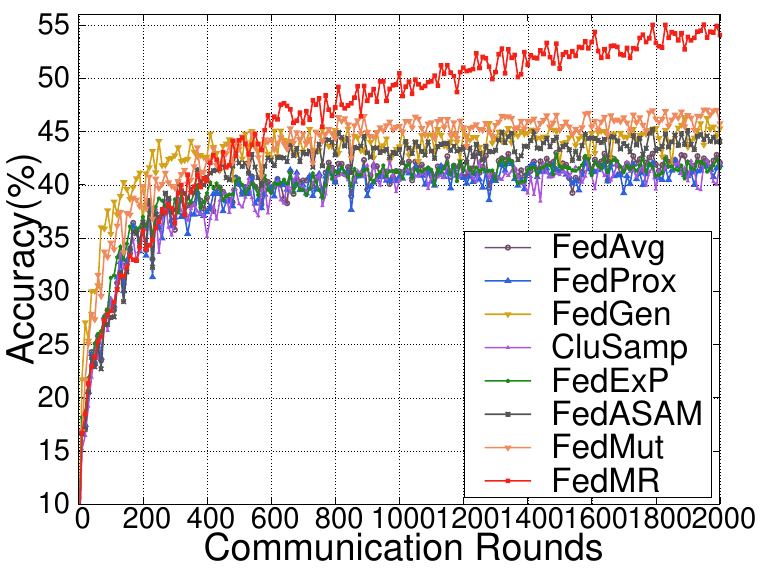}
		\label{fig:cifar-10-0.5}
	}\hspace{-0.1in}
	\subfigure[$\alpha=1.0$]{
		\centering
		\includegraphics[width=0.22\textwidth]{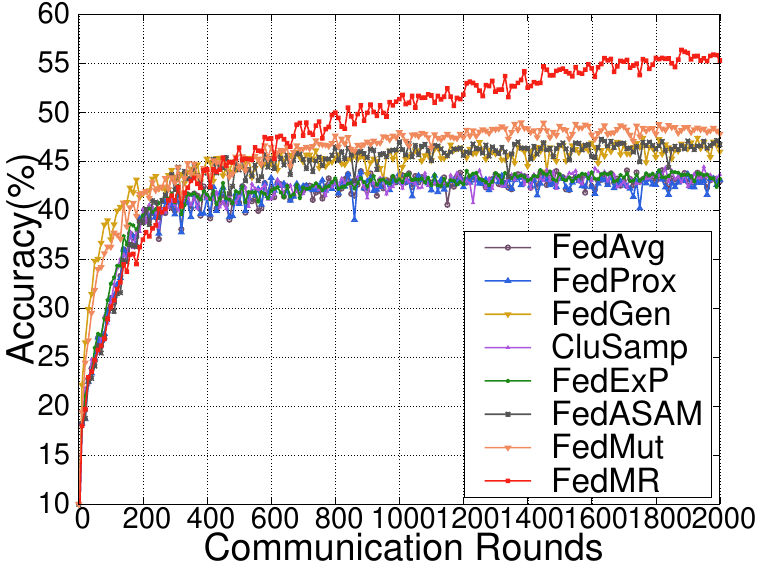}
		\label{fig:cifar-10-1.0}
	}\hspace{-0.1in}
	\subfigure[IID]{
		\centering
		\includegraphics[width=0.22\textwidth]{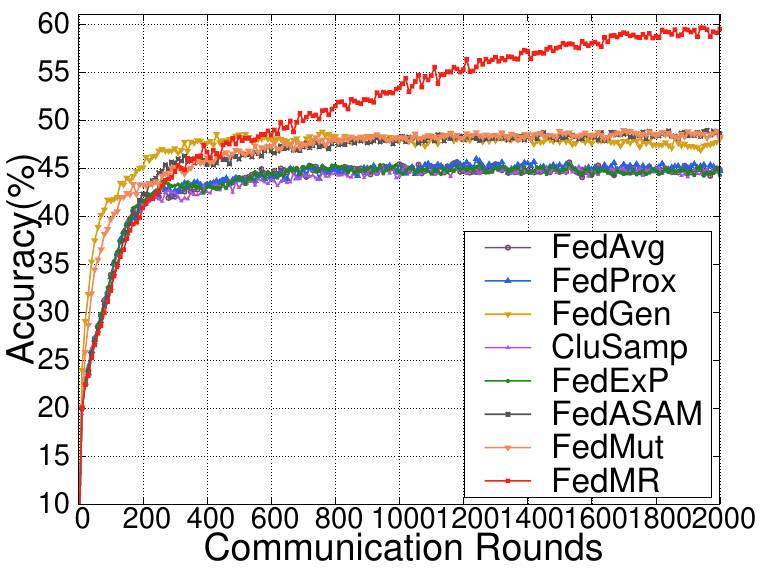}
		\label{fig:femnist}
	}
	\caption{Learning curves of FL methods based on  ResNet-20 model for CIFAR-100 dataset.}
	\label{fig:accuacy}
\end{figure*}

\begin{figure*}[t]
	\centering
	\subfigure[$K=5$]{
		\centering
		\hspace{-0.1in}
  \includegraphics[width=0.20\textwidth]{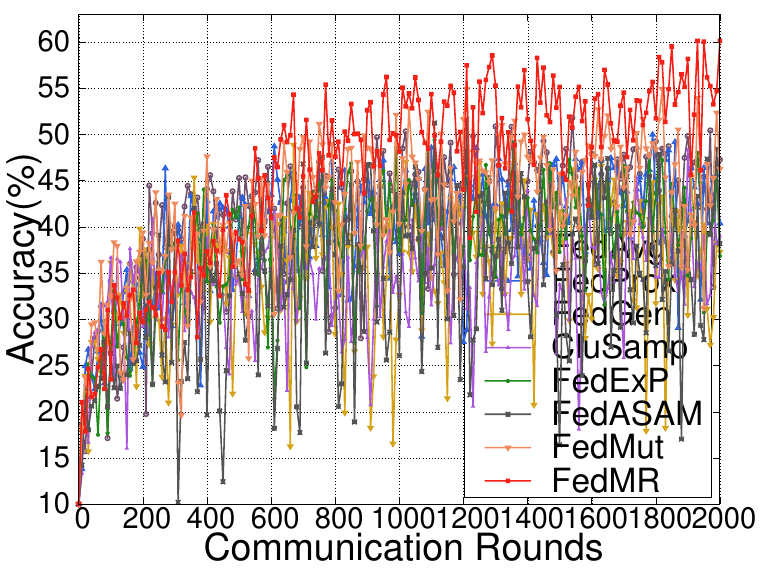}
		\label{fig:frac_0.05}
	}\hspace{-0.15in}
        \subfigure[$K=10$]{
		\centering
		\includegraphics[width=0.20\textwidth]{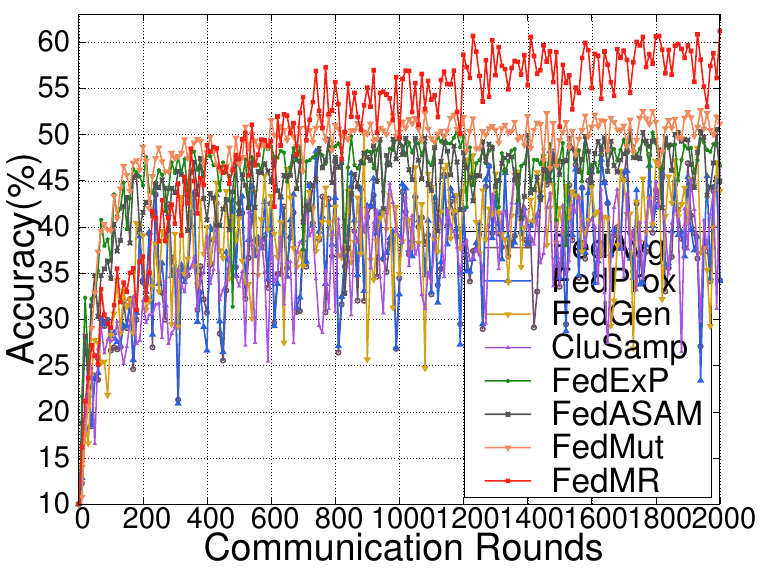}
		\label{fig:frac_0.1}
	}\hspace{-0.15in}
	\subfigure[$K=20$]{
		\centering
		\includegraphics[width=0.20\textwidth]{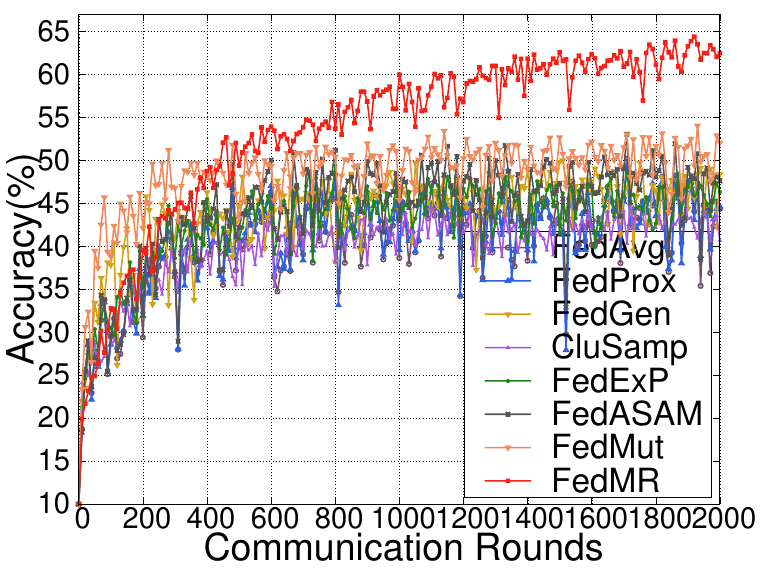}
		\label{fig:frac_0.2}
	}\hspace{-0.15in}
	\subfigure[$K=50$]{
		\centering
		\includegraphics[width=0.20\textwidth]{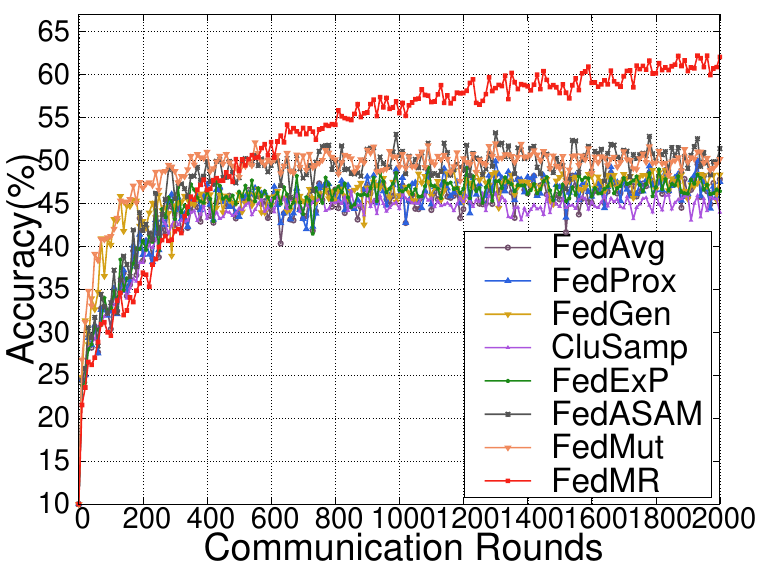}
		\label{fig:frac_0.5}
	}\hspace{-0.15in}
	\subfigure[$K=100$]{
		\centering
		\includegraphics[width=0.20\textwidth]{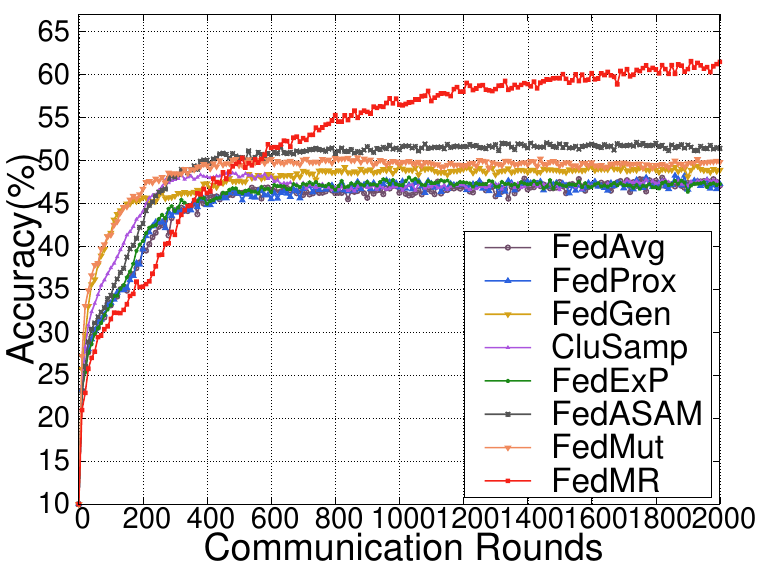}
		\label{fig:frac_1.0}
	}\hspace{-0.1in}
	\caption{Comparison  of  FL methods using ResNet-20 model on CIFAR-10 dataset with $\alpha=0.1$.}
	\label{fig:frac}
\end{figure*}


\subsection{Performance Comparison}

We compared the performance of our FedMR approach with seven SOTA baselines. 
For datasets CIFAR-10 and CIFAR-100, we considered both IID and non-IID scenarios (with $\alpha=0.1,0.5,1.0$, respectively). 
For dataset FEMNIST, we considered its original non-IID settings~\cite{leaf}.

\subsubsection{Comparison of Test Accuracy}
Table \ref{tab:acc} compares FedMR with the  SOTA FL methods, considering both non-IID and IID scenarios based on three different DL models. 
The first two columns denote the model type and dataset type, respectively. 
Note that to enable fair comparison, we cluster the test accuracy results generated by the FL methods based on the same type of local models. 
The third column shows different distribution settings for 
client data, indicating the data heterogeneity of clients. 
The fourth column has eight sub-columns, which present the test accuracy information together with its standard deviation
for all the investigated FL methods, respectively. 

From Table \ref{tab:acc}, we can observe that FedMR can achieve the highest test accuracy in all the scenarios regardless of model type, dataset type, and data heterogeneity. For CIFAR-10 and CIFAR-100, we can find that FedMR outperforms the seven baseline methods significantly in both non-IID and IID scenarios. 
For example, when dealing with a non-IID  CIFAR-10 scenario ($\alpha=0.1$) using ResNet-20-based models, FedMR achieves test accuracy with an average of 58.40\%, while the second highest average test accuracy obtained by FedMut is only 50.75\%.
Note that the performance of FedMR on FEMNIST is not as notable as the one on both CIFAR-10 and CIFAR-100. 
This is mainly because the classification task on FEMNIST is much simpler than the ones applied on datasets CIFAR-10 and CIFAR-100, which leads to the high test accuracy of the seven baseline methods. However, even in this case, FedMR can still achieve the best test accuracy among all the investigated FL methods.

\subsubsection{Comparison of Model Convergence}
Figure~\ref{fig:accuacy} presents the convergence trends 
of the seven FL methods (including FedMR) on the CIFAR-100 dataset.
Note that here the training of FedMR is based on our proposed 
two-stage training scheme, where the first stage uses 100 FL training rounds to achieve a pre-trained model. 
Here, to enable  fair comparison, the test accuracy of
 FedMR at some FL training rounds is calculated by an intermediate global model, which is an aggregated version of all the local models within that round. 
The four sub-figures show the results for different data distributions of clients. 
This figure shows that FedMR outperforms the other six FL methods consistently in both non-IID and IID scenarios.
This is mainly because FedMR can easily escape from the stuck-at-local-search due to the model recombination operation in each FL round. 
Moreover, due to the fine-grained training, we can observe that the learning curves in each sub-figure are much smoother than the ones of other FL methods. 
We also compared CNN- and VGG-16-based FL methods and found 
similar observations.  

\subsection{Ablation Study}


\subsubsection{Impacts of Activated Clients }
Figure \ref{fig:frac} compares
the learning trends between   FedMR and six baselines for a 
 non-IID scenario ($\alpha=0.1$) with both ResNet-20 model and 
 CIFAR-10 dataset, 
where the numbers of activated clients are $5$, $10$, $20$, $50$, and $100$, respectively.
From Figure~\ref{fig:frac}, we can observe that FedMR achieves the best inference performance
for all cases.
We can also observe that too few activated clients (i.e., $K=5$) result in a degradation of the accuracy of the global model in all the FL methods.
We find that when the number of activated clients increases, the convergence fluctuations reduce significantly and FedMR achieves the smallest fluctuations compared to all the baselines for all cases.
%



\begin{figure}[h]
\centering
	\subfigure[ResNet-20]{
		\centering
		\includegraphics[width=0.22\textwidth]{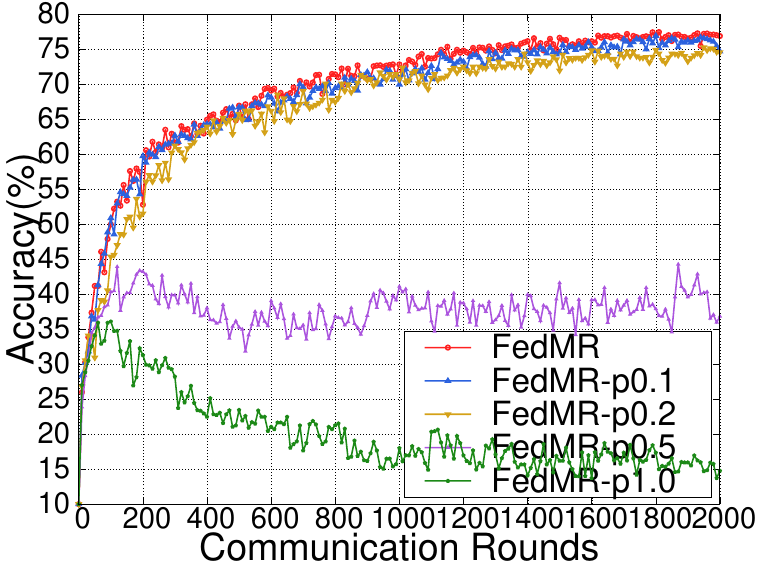}
		
	}\hspace{-0.1in}
    \subfigure[VGG-16]{
		\centering
		\includegraphics[width=0.22\textwidth]{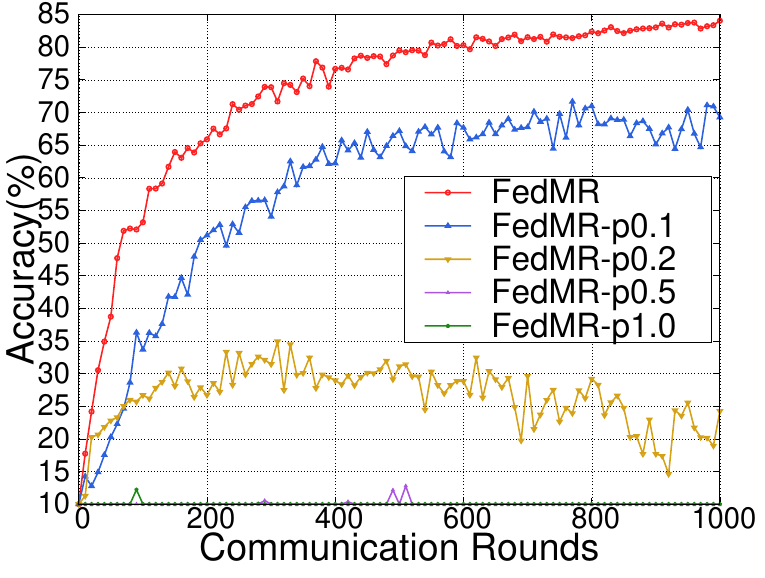}
	}\vspace{-0.15in}
 \caption{Learning curves for partitioning strategies.} 
 \label{fig:abl_mr}
\end{figure}

\subsubsection{Impacts of  Model Layer Partitioning} 
To show the effectiveness
of our layer-wise
model recombination scheme, we evaluated 
the FedMR performance using different 
model layer partitioning strategies. 
We use ``FedMR-p$x$'' to denote that the model is divided
into $\lceil\frac{1}{x}\rceil$ ($x\in(0,1.0]$) segments, where the model recombination is based on segments rather than layers. 
Note that FedMR does not divide a single layer into multiple segments. Instead, in FedMR each segment involves multiple layers (i.e., one or more layers), and the FedMR takes the recombination over segments among different clients. 
Specifically, for the $j^{th}$ layer, it belongs to $\lceil\frac{j}{x\times N}\rceil^{th}$ segment. Since $x\in(0,1.0]$, each segment contains at least one layer.
Note that $x=1.0$ indicates an extreme case, where local models are randomly dispatched to clients without recombination.

%
%

		

Figure~\ref{fig:abl_mr} presents the ablation study results on CIFAR-10 dataset
using ResNet-20-based and VGG-16-based FedMR, where the data on clients are non-IID distributed ($\alpha=1.0$).
Note that all the cases here did not use the two-stage training scheme. From this figure, we can find that FedMR outperforms the other variants significantly. Moreover, when the  
granularity of partitioning goes coarser, 
the classification performance of FedMR becomes 
worse. 


\begin{figure}[t]
\centering
    \subfigure[ResNet-20]{
		\centering
		\includegraphics[width=0.22\textwidth]{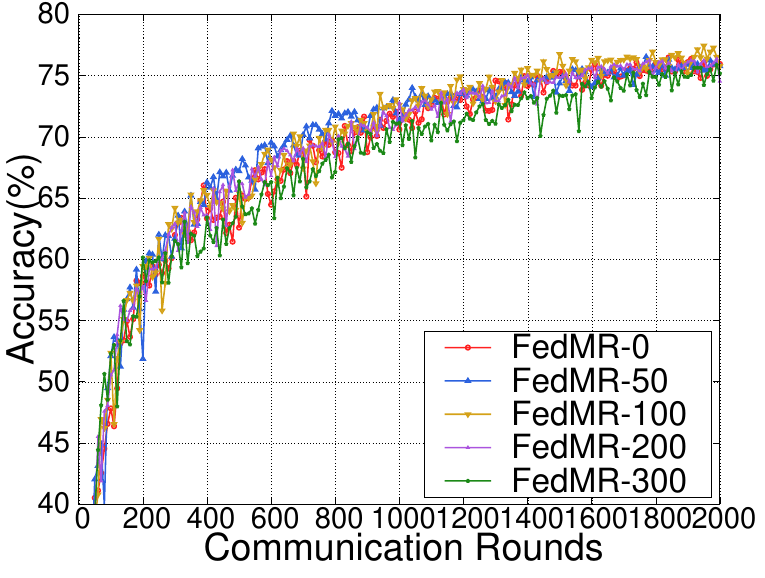}
	}\hspace{-0.1in}
	\subfigure[VGG-16]{
		\centering
		\includegraphics[width=0.22\textwidth]{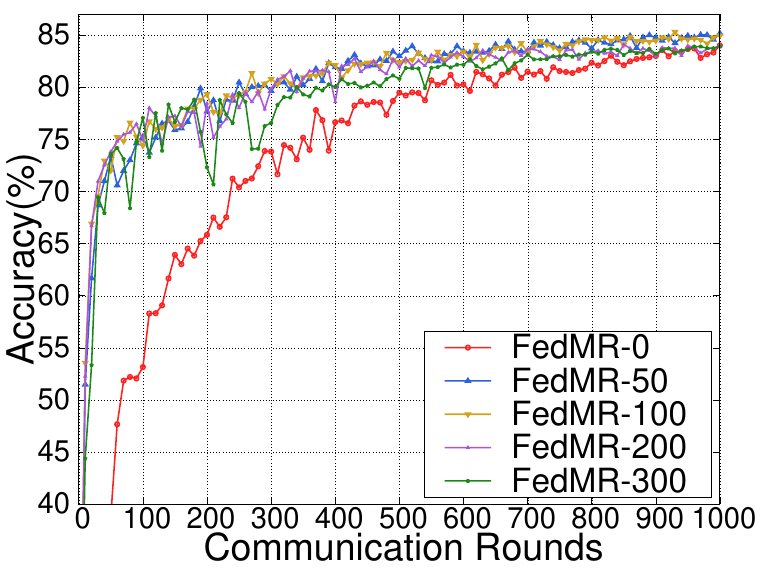}
	}
 \caption{Learning curves for  two-stage training settings.} \label{fig:abl_twostage}
\end{figure}

\subsubsection{Two-stage Training Scheme.} To demonstrate the effectiveness of our proposed two-stage training scheme, 
we conducted experiments on CIFAR-10 dataset using ResNet-20-based and VGG-16-based FedMR, 
where the data on clients are non-IID distributed ($\alpha=1.0$).  
Figure~\ref{fig:abl_twostage} presents the learning trends
of FedMR with five different two-stage training
 settings.
 Here,  we use the notation  ``FedMR-$n$''
 to denote that the first stage involves $n$ rounds of  model 
 aggregation-based local training to obtain a pre-trained global model, while 
 the remaining rounds conduct local training based on our proposed model recombination-based method. 
From Figure~\ref{fig:abl_twostage}, we can observe that the two-stage training-based
FedMR methods (i.e., FedMR-50 and FedMR-100) achieve 
the best performance from the perspectives of test accuracy and convergence rate.
Note that our two-stage training scheme can achieve a
more significant improvement  on the case using 
VGG-16 model, which has a much larger size than ResNet-20 model.
This is because the fine-grained FedMR without the first-stage training is not good at dealing with large-size models at the beginning of FL training, which requires many more training rounds than the  coarse-grained aggregation-based  methods to achieve a given preliminary
classification accuracy target. 
By resorting to the two-stage training scheme, 
such a slow convergence problem can be greatly 
mitigated.


\subsection{Discussions}
\subsubsection{Privacy Preserving} 
Similar to traditional FedAvg-based FL methods, FedMR does not require clients to send their data to the cloud server, thus the data privacy can be mostly guaranteed by the secure clients themselves. 
Since our model recombination operation breaks the dependencies between model layers and conducts the shuffling of layers among models, in practice, it is hard for adversaries to
restore the confidential data from a fragmentary recombined model without knowing the sources of layers.
We present a secure recombination mechanism to avoid privacy leakage from the cloud server. Please see Appendix \ref{sec:sec_mr} for more details. Our secure recombination mechanism ensures that the cloud server only receives recombined models from clients, which means that the cloud server cannot restore the model of each client.

\subsubsection{Limitations}
As a novel FL paradigm, FedMR shows much better inference
performance than most SOTA FL methods. Although this paper proposed an efficient two-stage training scheme to accelerate the overall FL training processes, there still exist numerous chances (e.g., client selection strategies,  dynamic combination of model aggregation and model recombination operations) to enable further optimization on the current version of FedMR. 
Meanwhile, the current version of FedMR does not consider personalization~\cite{pillutla2022federated,ma2022layer} and fairness~\cite{li2023fairer,li2023fairness}, two very important topics that deserve to be studied in the future.

\section{Conclusion}
Due to the coarse-grained aggregation of FedAvg as well as the  
uniform    client model initialization, when dealing with 
uneven data distribution among clients,
existing Federated Learning (FL) methods greatly suffer from 
the problem of low inference performance.
To address this problem, this paper presented a new FL paradigm named
FedMR, which enables different layers of local models to be trained on different clients.
Since FedMR supports
both fine-grained model recombination 
and diversified
local training
initialization, 
it enables effective and efficient search for superior generalized models for all clients. Comprehensive experimental results show both
the effectiveness and pervasiveness
of our proposed method in terms of inference accuracy and convergence rate.

\section{Acknowledgement}
This work was supported by Natural Science Foundation of China (62272170),   and ``Digital Silk Road'' Shanghai International Joint Lab of Trustworthy Intelligent Software (22510750100), the National Research Foundation, Singapore, and DSO National Laboratories under the AI Singapore Programme (AISG Award No: AISG2-GC-2023-008), the National Research Foundation, Singapore, and the Cyber Security Agency under its National Cybersecurity R\&D Programme (NCRP25-P04-TAICeN),  Natural Science Foundation (NSF CCF-2217104) and National Natural Science Foundation of China (62306116).
Any opinions, findings and conclusions or recommendations expressed in this material are those of the author(s) and do not reflect the views of National Research Foundation, Singapore and Cyber Security Agency of Singapore.

\bibliographystyle{ACM-Reference-Format}
\bibliography{sample-base}


\begin{thebibliography}{48}


\ifx \showCODEN    \undefined \def \showCODEN     #1{\unskip}     \fi
\ifx \showDOI      \undefined \def \showDOI       #1{#1}\fi
\ifx \showISBNx    \undefined \def \showISBNx     #1{\unskip}     \fi
\ifx \showISBNxiii \undefined \def \showISBNxiii  #1{\unskip}     \fi
\ifx \showISSN     \undefined \def \showISSN      #1{\unskip}     \fi
\ifx \showLCCN     \undefined \def \showLCCN      #1{\unskip}     \fi
\ifx \shownote     \undefined \def \shownote      #1{#1}          \fi
\ifx \showarticletitle \undefined \def \showarticletitle #1{#1}   \fi
\ifx \showURL      \undefined \def \showURL       {\relax}        \fi
\providecommand\bibfield[2]{#2}
\providecommand\bibinfo[2]{#2}
\providecommand\natexlab[1]{#1}
\providecommand\showeprint[2][]{arXiv:#2}

\bibitem[Acar et~al\mbox{.}(2020)]%
        {iclr_2021_durmus}
\bibfield{author}{\bibinfo{person}{Durmus Alp~Emre Acar}, \bibinfo{person}{Yue Zhao}, \bibinfo{person}{Ramon Matas}, \bibinfo{person}{Matthew Mattina}, \bibinfo{person}{Paul Whatmough}, {and} \bibinfo{person}{Venkatesh Saligrama}.} \bibinfo{year}{2020}\natexlab{}.
\newblock \showarticletitle{Federated Learning Based on Dynamic Regularization}. In \bibinfo{booktitle}{\emph{Proc. of International Conference on Learning Representations}}.
\newblock


\bibitem[Adnan et~al\mbox{.}(2022)]%
        {adnan2022federated}
\bibfield{author}{\bibinfo{person}{Mohammed Adnan}, \bibinfo{person}{Shivam Kalra}, \bibinfo{person}{Jesse~C Cresswell}, \bibinfo{person}{Graham~W Taylor}, {and} \bibinfo{person}{Hamid~R Tizhoosh}.} \bibinfo{year}{2022}\natexlab{}.
\newblock \showarticletitle{Federated learning and differential privacy for medical image analysis}.
\newblock \bibinfo{journal}{\emph{Scientific reports}} \bibinfo{volume}{12}, \bibinfo{number}{1} (\bibinfo{year}{2022}), \bibinfo{pages}{1953}.
\newblock


\bibitem[Agarwal et~al\mbox{.}(2021)]%
        {nips_naman_2021}
\bibfield{author}{\bibinfo{person}{Naman Agarwal}, \bibinfo{person}{Peter Kairouz}, {and} \bibinfo{person}{Ziyu Liu}.} \bibinfo{year}{2021}\natexlab{}.
\newblock \showarticletitle{The skellam mechanism for differentially private federated learning}.
\newblock \bibinfo{journal}{\emph{Advances in Neural Information Processing Systems}}  \bibinfo{volume}{34} (\bibinfo{year}{2021}), \bibinfo{pages}{5052--5064}.
\newblock


\bibitem[Alam et~al\mbox{.}(2022)]%
        {alam2022fedrolex}
\bibfield{author}{\bibinfo{person}{Samiul Alam}, \bibinfo{person}{Luyang Liu}, \bibinfo{person}{Ming Yan}, {and} \bibinfo{person}{Mi Zhang}.} \bibinfo{year}{2022}\natexlab{}.
\newblock \showarticletitle{Fedrolex: Model-heterogeneous federated learning with rolling sub-model extraction}.
\newblock \bibinfo{journal}{\emph{Advances in neural information processing systems}}  \bibinfo{volume}{35} (\bibinfo{year}{2022}), \bibinfo{pages}{29677--29690}.
\newblock


\bibitem[Bonawitz et~al\mbox{.}(2017)]%
        {bonawitz2017practical}
\bibfield{author}{\bibinfo{person}{Keith Bonawitz}, \bibinfo{person}{Vladimir Ivanov}, \bibinfo{person}{Ben Kreuter}, \bibinfo{person}{Antonio Marcedone}, \bibinfo{person}{H~Brendan McMahan}, \bibinfo{person}{Sarvar Patel}, \bibinfo{person}{Daniel Ramage}, \bibinfo{person}{Aaron Segal}, {and} \bibinfo{person}{Karn Seth}.} \bibinfo{year}{2017}\natexlab{}.
\newblock \showarticletitle{Practical secure aggregation for privacy-preserving machine learning}. In \bibinfo{booktitle}{\emph{Proc. of ACM SIGSAC Conference on Computer and Communications Security}}. \bibinfo{pages}{1175--1191}.
\newblock


\bibitem[Caldarola et~al\mbox{.}(2022)]%
        {fedasam}
\bibfield{author}{\bibinfo{person}{Debora Caldarola}, \bibinfo{person}{Barbara Caputo}, {and} \bibinfo{person}{Marco Ciccone}.} \bibinfo{year}{2022}\natexlab{}.
\newblock \showarticletitle{Improving generalization in federated learning by seeking flat minima}. In \bibinfo{booktitle}{\emph{Proc. of European Conference on Computer Vision}}. \bibinfo{pages}{654--672}.
\newblock


\bibitem[Caldas et~al\mbox{.}(2018)]%
        {leaf}
\bibfield{author}{\bibinfo{person}{Sebastian Caldas}, \bibinfo{person}{Sai Meher~Karthik Duddu}, \bibinfo{person}{Peter Wu}, \bibinfo{person}{Tian Li}, \bibinfo{person}{Jakub Kone{\v{c}}n{\`y}}, \bibinfo{person}{H~Brendan McMahan}, \bibinfo{person}{Virginia Smith}, {and} \bibinfo{person}{Ameet Talwalkar}.} \bibinfo{year}{2018}\natexlab{}.
\newblock \showarticletitle{Leaf: A benchmark for federated settings}.
\newblock \bibinfo{journal}{\emph{arXiv preprint arXiv:1812.01097}} (\bibinfo{year}{2018}).
\newblock


\bibitem[Chen et~al\mbox{.}(2020)]%
        {bigdata_cheng_2020}
\bibfield{author}{\bibinfo{person}{Cheng Chen}, \bibinfo{person}{Ziyi Chen}, \bibinfo{person}{Yi Zhou}, {and} \bibinfo{person}{Bhavya Kailkhura}.} \bibinfo{year}{2020}\natexlab{}.
\newblock \showarticletitle{Fedcluster: Boosting the convergence of federated learning via cluster-cycling}. In \bibinfo{booktitle}{\emph{Proc. of IEEE International Conference on Big Data}}. \bibinfo{pages}{5017--5026}.
\newblock


\bibitem[Fraboni et~al\mbox{.}(2021)]%
        {icml_yann_2021}
\bibfield{author}{\bibinfo{person}{Yann Fraboni}, \bibinfo{person}{Richard Vidal}, \bibinfo{person}{Laetitia Kameni}, {and} \bibinfo{person}{Marco Lorenzi}.} \bibinfo{year}{2021}\natexlab{}.
\newblock \showarticletitle{Clustered sampling: Low-variance and improved representativity for clients selection in federated learning}. In \bibinfo{booktitle}{\emph{Proc. of International Conference on Machine Learning}}. \bibinfo{pages}{3407--3416}.
\newblock


\bibitem[Hardt et~al\mbox{.}(2016)]%
        {hardt2016train}
\bibfield{author}{\bibinfo{person}{Moritz Hardt}, \bibinfo{person}{Ben Recht}, {and} \bibinfo{person}{Yoram Singer}.} \bibinfo{year}{2016}\natexlab{}.
\newblock \showarticletitle{Train faster, generalize better: Stability of stochastic gradient descent}. In \bibinfo{booktitle}{\emph{Proc. of International Conference on Machine Learning}}. \bibinfo{pages}{1225--1234}.
\newblock


\bibitem[Hochreiter and Schmidhuber(1997)]%
        {hochreiter1997flat}
\bibfield{author}{\bibinfo{person}{Sepp Hochreiter} {and} \bibinfo{person}{J{\"u}rgen Schmidhuber}.} \bibinfo{year}{1997}\natexlab{}.
\newblock \showarticletitle{Flat minima}.
\newblock \bibinfo{journal}{\emph{Neural computation}} \bibinfo{volume}{9}, \bibinfo{number}{1} (\bibinfo{year}{1997}), \bibinfo{pages}{1--42}.
\newblock


\bibitem[Hsu et~al\mbox{.}(2019)]%
        {measuring}
\bibfield{author}{\bibinfo{person}{Tzu-Ming~Harry Hsu}, \bibinfo{person}{Hang Qi}, {and} \bibinfo{person}{Matthew Brown}.} \bibinfo{year}{2019}\natexlab{}.
\newblock \showarticletitle{Measuring the effects of non-identical data distribution for federated visual classification}.
\newblock \bibinfo{journal}{\emph{arXiv preprint arXiv:1909.06335}} (\bibinfo{year}{2019}).
\newblock


\bibitem[Hu et~al\mbox{.}(2023a)]%
        {hu2023aiotml}
\bibfield{author}{\bibinfo{person}{Ming Hu}, \bibinfo{person}{E Cao}, \bibinfo{person}{Hongbing Huang}, \bibinfo{person}{Min Zhang}, \bibinfo{person}{Xiaohong Chen}, {and} \bibinfo{person}{Mingsong Chen}.} \bibinfo{year}{2023}\natexlab{a}.
\newblock \showarticletitle{AIoTML: A Unified Modeling Language for AIoT-Based Cyber-Physical Systems}.
\newblock \bibinfo{journal}{\emph{IEEE Transactions on Computer-Aided Design of Integrated Circuits and Systems}} (\bibinfo{year}{2023}).
\newblock


\bibitem[Hu et~al\mbox{.}(2024)]%
        {hu2024fedmut}
\bibfield{author}{\bibinfo{person}{Ming Hu}, \bibinfo{person}{Yue Cao}, \bibinfo{person}{Anran Li}, \bibinfo{person}{Zhiming Li}, \bibinfo{person}{Chengwei Liu}, \bibinfo{person}{Tianlin Li}, \bibinfo{person}{Mingsong Chen}, {and} \bibinfo{person}{Yang Liu}.} \bibinfo{year}{2024}\natexlab{}.
\newblock \showarticletitle{FedMut: Generalized Federated Learning via Stochastic Mutation}. In \bibinfo{booktitle}{\emph{Proc. of the AAAI Conference on Artificial Intelligence}}, Vol.~\bibinfo{volume}{38}. \bibinfo{pages}{12528--12537}.
\newblock


\bibitem[Hu et~al\mbox{.}(2020)]%
        {hu2020quantitative}
\bibfield{author}{\bibinfo{person}{Ming Hu}, \bibinfo{person}{Wenxue Duan}, \bibinfo{person}{Min Zhang}, \bibinfo{person}{Tongquan Wei}, {and} \bibinfo{person}{Mingsong Chen}.} \bibinfo{year}{2020}\natexlab{}.
\newblock \showarticletitle{Quantitative timing analysis for cyber-physical systems using uncertainty-aware scenario-based specifications}.
\newblock \bibinfo{journal}{\emph{IEEE Transactions on Computer-Aided Design of Integrated Circuits and Systems}} \bibinfo{volume}{39}, \bibinfo{number}{11} (\bibinfo{year}{2020}), \bibinfo{pages}{4006--4017}.
\newblock


\bibitem[Hu et~al\mbox{.}(2023b)]%
        {hu2023gitfl}
\bibfield{author}{\bibinfo{person}{Ming Hu}, \bibinfo{person}{Zeke Xia}, \bibinfo{person}{Dengke Yan}, \bibinfo{person}{Zhihao Yue}, \bibinfo{person}{Jun Xia}, \bibinfo{person}{Yihao Huang}, \bibinfo{person}{Yang Liu}, {and} \bibinfo{person}{Mingsong Chen}.} \bibinfo{year}{2023}\natexlab{b}.
\newblock \showarticletitle{GitFL: Uncertainty-Aware Real-Time Asynchronous Federated Learning Using Version Control}. In \bibinfo{booktitle}{\emph{Proc. of IEEE Real-Time Systems Symposium}}. \bibinfo{pages}{145--157}.
\newblock


\bibitem[Huang et~al\mbox{.}(2021)]%
        {aaai_yutao_2021}
\bibfield{author}{\bibinfo{person}{Yutao Huang}, \bibinfo{person}{Lingyang Chu}, \bibinfo{person}{Zirui Zhou}, \bibinfo{person}{Lanjun Wang}, \bibinfo{person}{Jiangchuan Liu}, \bibinfo{person}{Jian Pei}, {and} \bibinfo{person}{Yong Zhang}.} \bibinfo{year}{2021}\natexlab{}.
\newblock \showarticletitle{Personalized Cross-Silo Federated Learning on Non-IID Data}. In \bibinfo{booktitle}{\emph{Proc. of the {AAAI} Conference on Artificial Intelligence}}.
\newblock


\bibitem[Ilhan et~al\mbox{.}(2023)]%
        {ilhan2023scalefl}
\bibfield{author}{\bibinfo{person}{Fatih Ilhan}, \bibinfo{person}{Gong Su}, {and} \bibinfo{person}{Ling Liu}.} \bibinfo{year}{2023}\natexlab{}.
\newblock \showarticletitle{Scalefl: Resource-adaptive federated learning with heterogeneous clients}. In \bibinfo{booktitle}{\emph{Proc. of the IEEE/CVF Conference on Computer Vision and Pattern Recognition}}. \bibinfo{pages}{24532--24541}.
\newblock


\bibitem[Jhunjhunwala et~al\mbox{.}(2023)]%
        {jhunjhunwala2023fedexp}
\bibfield{author}{\bibinfo{person}{Divyansh Jhunjhunwala}, \bibinfo{person}{Shiqiang Wang}, {and} \bibinfo{person}{Gauri Joshi}.} \bibinfo{year}{2023}\natexlab{}.
\newblock \showarticletitle{Fedexp: Speeding up federated averaging via extrapolation}.
\newblock \bibinfo{journal}{\emph{arXiv preprint arXiv:2301.09604}} (\bibinfo{year}{2023}).
\newblock


\bibitem[Jia et~al\mbox{.}(2024)]%
        {jia2023adaptivefl}
\bibfield{author}{\bibinfo{person}{Chentao Jia}, \bibinfo{person}{Ming Hu}, \bibinfo{person}{Zekai Chen}, \bibinfo{person}{Yanxin Yang}, \bibinfo{person}{Xiaofei Xie}, \bibinfo{person}{Yang Liu}, {and} \bibinfo{person}{Mingsong Chen}.} \bibinfo{year}{2024}\natexlab{}.
\newblock \showarticletitle{AdaptiveFL: Adaptive Heterogeneous Federated Learning for Resource-Constrained AIoT Systems}. In \bibinfo{booktitle}{\emph{Proc. of Design Automation Conference}}. \bibinfo{pages}{1--6}.
\newblock


\bibitem[Kairouz et~al\mbox{.}(2021)]%
        {kairouz2021advances}
\bibfield{author}{\bibinfo{person}{Peter Kairouz}, \bibinfo{person}{H~Brendan McMahan}, \bibinfo{person}{Brendan Avent}, \bibinfo{person}{Aur{\'e}lien Bellet}, \bibinfo{person}{Mehdi Bennis}, \bibinfo{person}{Arjun~Nitin Bhagoji}, \bibinfo{person}{Kallista Bonawitz}, \bibinfo{person}{Zachary Charles}, \bibinfo{person}{Graham Cormode}, \bibinfo{person}{Rachel Cummings}, {et~al\mbox{.}}} \bibinfo{year}{2021}\natexlab{}.
\newblock \showarticletitle{Advances and open problems in federated learning}.
\newblock \bibinfo{journal}{\emph{Foundations and Trends{\textregistered} in Machine Learning}} \bibinfo{volume}{14}, \bibinfo{number}{1--2} (\bibinfo{year}{2021}), \bibinfo{pages}{1--210}.
\newblock


\bibitem[Karimireddy et~al\mbox{.}(2020)]%
        {icml_scaffold}
\bibfield{author}{\bibinfo{person}{Sai~Praneeth Karimireddy}, \bibinfo{person}{Satyen Kale}, \bibinfo{person}{Mehryar Mohri}, \bibinfo{person}{Sashank Reddi}, \bibinfo{person}{Sebastian Stich}, {and} \bibinfo{person}{Ananda~Theertha Suresh}.} \bibinfo{year}{2020}\natexlab{}.
\newblock \showarticletitle{SCAFFOLD: Stochastic controlled averaging for federated learning}. In \bibinfo{booktitle}{\emph{Proc. of International Conference on Machine Learning}}. \bibinfo{pages}{5132--5143}.
\newblock


\bibitem[Krizhevsky(2009)]%
        {data}
\bibfield{author}{\bibinfo{person}{Alex Krizhevsky}.} \bibinfo{year}{2009}\natexlab{}.
\newblock \showarticletitle{Learning Multiple Layers of Features from Tiny Images}.
\newblock \bibinfo{journal}{\emph{Master's thesis, University of Tront}} (\bibinfo{year}{2009}).
\newblock


\bibitem[Kwon et~al\mbox{.}(2021)]%
        {kwon2021asam}
\bibfield{author}{\bibinfo{person}{Jungmin Kwon}, \bibinfo{person}{Jeongseop Kim}, \bibinfo{person}{Hyunseo Park}, {and} \bibinfo{person}{In~Kwon Choi}.} \bibinfo{year}{2021}\natexlab{}.
\newblock \showarticletitle{Asam: Adaptive sharpness-aware minimization for scale-invariant learning of deep neural networks}. In \bibinfo{booktitle}{\emph{Proc. of International Conference on Machine Learning}}. PMLR, \bibinfo{pages}{5905--5914}.
\newblock


\bibitem[Li et~al\mbox{.}(2023c)]%
        {li2023towards}
\bibfield{author}{\bibinfo{person}{Anran Li}, \bibinfo{person}{Rui Liu}, \bibinfo{person}{Ming Hu}, \bibinfo{person}{Luu~Anh Tuan}, {and} \bibinfo{person}{Han Yu}.} \bibinfo{year}{2023}\natexlab{c}.
\newblock \showarticletitle{Towards interpretable federated learning}.
\newblock \bibinfo{journal}{\emph{arXiv preprint arXiv:2302.13473}} (\bibinfo{year}{2023}).
\newblock


\bibitem[Li et~al\mbox{.}(2023a)]%
        {li2023fairer}
\bibfield{author}{\bibinfo{person}{Tianlin Li}, \bibinfo{person}{Qing Guo}, \bibinfo{person}{Aishan Liu}, \bibinfo{person}{Mengnan Du}, \bibinfo{person}{Zhiming Li}, {and} \bibinfo{person}{Yang Liu}.} \bibinfo{year}{2023}\natexlab{a}.
\newblock \showarticletitle{FAIRER: fairness as decision rationale alignment}. In \bibinfo{booktitle}{\emph{Proc. of International Conference on Machine Learning}}. \bibinfo{pages}{19471--19489}.
\newblock


\bibitem[Li et~al\mbox{.}(2023b)]%
        {li2023fairness}
\bibfield{author}{\bibinfo{person}{Tianlin Li}, \bibinfo{person}{Zhiming Li}, \bibinfo{person}{Anran Li}, \bibinfo{person}{Mengnan Du}, \bibinfo{person}{Aishan Liu}, \bibinfo{person}{Qing Guo}, \bibinfo{person}{Guozhu Meng}, {and} \bibinfo{person}{Yang Liu}.} \bibinfo{year}{2023}\natexlab{b}.
\newblock \showarticletitle{Fairness via group contribution matching}. In \bibinfo{booktitle}{\emph{Proc. of International Joint Conference on Artificial Intelligence}}. \bibinfo{pages}{436--445}.
\newblock


\bibitem[Li et~al\mbox{.}(2020b)]%
        {fedprox}
\bibfield{author}{\bibinfo{person}{Tian Li}, \bibinfo{person}{Anit~Kumar Sahu}, \bibinfo{person}{Manzil Zaheer}, \bibinfo{person}{Maziar Sanjabi}, \bibinfo{person}{Ameet Talwalkar}, {and} \bibinfo{person}{Virginia Smith}.} \bibinfo{year}{2020}\natexlab{b}.
\newblock \showarticletitle{Federated optimization in heterogeneous networks}.
\newblock \bibinfo{journal}{\emph{Proc. of Machine Learning and Systems}}  \bibinfo{volume}{2} (\bibinfo{year}{2020}), \bibinfo{pages}{429--450}.
\newblock


\bibitem[Li et~al\mbox{.}(2020a)]%
        {convergence}
\bibfield{author}{\bibinfo{person}{Xiang Li}, \bibinfo{person}{Kaixuan Huang}, \bibinfo{person}{Wenhao Yang}, \bibinfo{person}{Shusen Wang}, {and} \bibinfo{person}{Zhihua Zhang}.} \bibinfo{year}{2020}\natexlab{a}.
\newblock \showarticletitle{On the Convergence of FedAvg on Non-IID Data}. In \bibinfo{booktitle}{\emph{Proc. of International Conference on Learning Representations}}.
\newblock


\bibitem[Lin et~al\mbox{.}(2020)]%
        {nips_tao_2020}
\bibfield{author}{\bibinfo{person}{Tao Lin}, \bibinfo{person}{Lingjing Kong}, \bibinfo{person}{Sebastian~U Stich}, {and} \bibinfo{person}{Martin Jaggi}.} \bibinfo{year}{2020}\natexlab{}.
\newblock \showarticletitle{Ensemble distillation for robust model fusion in federated learning}.
\newblock \bibinfo{journal}{\emph{Advances in Neural Information Processing Systems}}  \bibinfo{volume}{33} (\bibinfo{year}{2020}), \bibinfo{pages}{2351--2363}.
\newblock


\bibitem[Ma et~al\mbox{.}(2022)]%
        {ma2022layer}
\bibfield{author}{\bibinfo{person}{Xiaosong Ma}, \bibinfo{person}{Jie Zhang}, \bibinfo{person}{Song Guo}, {and} \bibinfo{person}{Wenchao Xu}.} \bibinfo{year}{2022}\natexlab{}.
\newblock \showarticletitle{Layer-wised model aggregation for personalized federated learning}. In \bibinfo{booktitle}{\emph{Proc. of the IEEE/CVF conference on computer vision and pattern recognition}}. \bibinfo{pages}{10092--10101}.
\newblock


\bibitem[McMahan et~al\mbox{.}(2017)]%
        {fedavg}
\bibfield{author}{\bibinfo{person}{Brendan McMahan}, \bibinfo{person}{Eider Moore}, \bibinfo{person}{Daniel Ramage}, \bibinfo{person}{Seth Hampson}, {and} \bibinfo{person}{Blaise~Aguera y Arcas}.} \bibinfo{year}{2017}\natexlab{}.
\newblock \showarticletitle{Communication-efficient learning of deep networks from decentralized data}. In \bibinfo{booktitle}{\emph{Proc. of Artificial intelligence and statistics}}. \bibinfo{pages}{1273--1282}.
\newblock


\bibitem[Pillutla et~al\mbox{.}(2022)]%
        {pillutla2022federated}
\bibfield{author}{\bibinfo{person}{Krishna Pillutla}, \bibinfo{person}{Kshitiz Malik}, \bibinfo{person}{Abdel-Rahman Mohamed}, \bibinfo{person}{Mike Rabbat}, \bibinfo{person}{Maziar Sanjabi}, {and} \bibinfo{person}{Lin Xiao}.} \bibinfo{year}{2022}\natexlab{}.
\newblock \showarticletitle{Federated learning with partial model personalization}. In \bibinfo{booktitle}{\emph{Proc. of International Conference on Machine Learning}}. \bibinfo{pages}{17716--17758}.
\newblock


\bibitem[Stich(2019)]%
        {sgd}
\bibfield{author}{\bibinfo{person}{Sebastian~Urban Stich}.} \bibinfo{year}{2019}\natexlab{}.
\newblock \showarticletitle{Local SGD Converges Fast and Communicates Little}. In \bibinfo{booktitle}{\emph{Proc. of International Conference on Learning Representations}}.
\newblock


\bibitem[Tan et~al\mbox{.}(2020)]%
        {tan2020federated}
\bibfield{author}{\bibinfo{person}{Ben Tan}, \bibinfo{person}{Bo Liu}, \bibinfo{person}{Vincent Zheng}, {and} \bibinfo{person}{Qiang Yang}.} \bibinfo{year}{2020}\natexlab{}.
\newblock \showarticletitle{A federated recommender system for online services}. In \bibinfo{booktitle}{\emph{Proc. of ACM Conference on Recommender Systems}}. \bibinfo{pages}{579--581}.
\newblock


\bibitem[TorchvisionModel(2019)]%
        {models}
\bibfield{author}{\bibinfo{person}{TorchvisionModel}.} \bibinfo{year}{2019}\natexlab{}.
\newblock \bibinfo{title}{Models and pre-trained weight}.
\newblock \bibinfo{howpublished}{\url{https://pytorch.org/vision/stable/models.html}}.
\newblock


\bibitem[Wang et~al\mbox{.}(2024)]%
        {wang2024feddse}
\bibfield{author}{\bibinfo{person}{Haozhao Wang}, \bibinfo{person}{Yabo Jia}, \bibinfo{person}{Meng Zhang}, \bibinfo{person}{Qinghao Hu}, \bibinfo{person}{Hao Ren}, \bibinfo{person}{Peng Sun}, \bibinfo{person}{Yonggang Wen}, {and} \bibinfo{person}{Tianwei Zhang}.} \bibinfo{year}{2024}\natexlab{}.
\newblock \showarticletitle{FedDSE: Distribution-aware Sub-model Extraction for Federated Learning over Resource-constrained Devices}. In \bibinfo{booktitle}{\emph{Proc. of the ACM on Web Conference}}. \bibinfo{pages}{2902--2913}.
\newblock


\bibitem[Wang et~al\mbox{.}(2020)]%
        {infocom_hao_2020}
\bibfield{author}{\bibinfo{person}{Hao Wang}, \bibinfo{person}{Zakhary Kaplan}, \bibinfo{person}{Di Niu}, {and} \bibinfo{person}{Baochun Li}.} \bibinfo{year}{2020}\natexlab{}.
\newblock \showarticletitle{Optimizing federated learning on non-iid data with reinforcement learning}. In \bibinfo{booktitle}{\emph{Proc. of IEEE Conference on Computer Communications}}. \bibinfo{pages}{1698--1707}.
\newblock


\bibitem[Wang et~al\mbox{.}(2023b)]%
        {wang2023dafkd}
\bibfield{author}{\bibinfo{person}{Haozhao Wang}, \bibinfo{person}{Yichen Li}, \bibinfo{person}{Wenchao Xu}, \bibinfo{person}{Ruixuan Li}, \bibinfo{person}{Yufeng Zhan}, {and} \bibinfo{person}{Zhigang Zeng}.} \bibinfo{year}{2023}\natexlab{b}.
\newblock \showarticletitle{Dafkd: Domain-aware federated knowledge distillation}. In \bibinfo{booktitle}{\emph{Proc. of the IEEE/CVF conference on Computer Vision and Pattern Recognition}}. \bibinfo{pages}{20412--20421}.
\newblock


\bibitem[Wang et~al\mbox{.}(2023c)]%
        {wang2023fedcda}
\bibfield{author}{\bibinfo{person}{Haozhao Wang}, \bibinfo{person}{Haoran Xu}, \bibinfo{person}{Yichen Li}, \bibinfo{person}{Yuan Xu}, \bibinfo{person}{Ruixuan Li}, {and} \bibinfo{person}{Tianwei Zhang}.} \bibinfo{year}{2023}\natexlab{c}.
\newblock \showarticletitle{FedCDA: Federated Learning with Cross-rounds Divergence-aware Aggregation}. In \bibinfo{booktitle}{\emph{Proc. of International Conference on Learning Representations}}.
\newblock


\bibitem[Wang et~al\mbox{.}(2023a)]%
        {wang2023fedlego}
\bibfield{author}{\bibinfo{person}{Jiaqi Wang}, \bibinfo{person}{Suhan Cui}, {and} \bibinfo{person}{Fenglong Ma}.} \bibinfo{year}{2023}\natexlab{a}.
\newblock \showarticletitle{FedLEGO: Enabling Heterogenous Model Cooperation via Brick Reassembly in Federated Learning}. In \bibinfo{booktitle}{\emph{Proc. of International Workshop on Federated Learning for Distributed Data Mining}}.
\newblock


\bibitem[Wen et~al\mbox{.}(2023)]%
        {wen2023survey}
\bibfield{author}{\bibinfo{person}{Jie Wen}, \bibinfo{person}{Zhixia Zhang}, \bibinfo{person}{Yang Lan}, \bibinfo{person}{Zhihua Cui}, \bibinfo{person}{Jianghui Cai}, {and} \bibinfo{person}{Wensheng Zhang}.} \bibinfo{year}{2023}\natexlab{}.
\newblock \showarticletitle{A survey on federated learning: challenges and applications}.
\newblock \bibinfo{journal}{\emph{International Journal of Machine Learning and Cybernetics}} \bibinfo{volume}{14}, \bibinfo{number}{2} (\bibinfo{year}{2023}), \bibinfo{pages}{513--535}.
\newblock


\bibitem[Wu et~al\mbox{.}(2020)]%
        {wu2020adversarial}
\bibfield{author}{\bibinfo{person}{Dongxian Wu}, \bibinfo{person}{Shu-Tao Xia}, {and} \bibinfo{person}{Yisen Wang}.} \bibinfo{year}{2020}\natexlab{}.
\newblock \showarticletitle{Adversarial weight perturbation helps robust generalization}.
\newblock \bibinfo{journal}{\emph{Advances in Neural Information Processing Systems}}  \bibinfo{volume}{33} (\bibinfo{year}{2020}), \bibinfo{pages}{2958--2969}.
\newblock


\bibitem[Xia et~al\mbox{.}(2024)]%
        {xia2024cabafl}
\bibfield{author}{\bibinfo{person}{Zeke Xia}, \bibinfo{person}{Ming Hu}, \bibinfo{person}{Dengke Yan}, \bibinfo{person}{Xiaofei Xie}, \bibinfo{person}{Tianlin Li}, \bibinfo{person}{Anran Li}, \bibinfo{person}{Junlong Zhou}, {and} \bibinfo{person}{Mingsong Chen}.} \bibinfo{year}{2024}\natexlab{}.
\newblock \showarticletitle{CaBaFL: Asynchronous Federated Learning via Hierarchical Cache and Feature Balance}.
\newblock \bibinfo{journal}{\emph{arXiv preprint arXiv:2404.12850}} (\bibinfo{year}{2024}).
\newblock


\bibitem[Xie et~al\mbox{.}(2020)]%
        {axiv_ming_2021}
\bibfield{author}{\bibinfo{person}{Ming Xie}, \bibinfo{person}{Guodong Long}, \bibinfo{person}{Tao Shen}, \bibinfo{person}{Tianyi Zhou}, \bibinfo{person}{Xianzhi Wang}, \bibinfo{person}{Jing Jiang}, {and} \bibinfo{person}{Chengqi Zhang}.} \bibinfo{year}{2020}\natexlab{}.
\newblock \showarticletitle{Multi-center federated learning}.
\newblock \bibinfo{journal}{\emph{arXiv preprint arXiv:2005.01026}} (\bibinfo{year}{2020}).
\newblock


\bibitem[Yang et~al\mbox{.}(2021)]%
        {kdd_qian_2021}
\bibfield{author}{\bibinfo{person}{Qian Yang}, \bibinfo{person}{Jianyi Zhang}, \bibinfo{person}{Weituo Hao}, \bibinfo{person}{Gregory~P Spell}, {and} \bibinfo{person}{Lawrence Carin}.} \bibinfo{year}{2021}\natexlab{}.
\newblock \showarticletitle{Flop: Federated learning on medical datasets using partial networks}. In \bibinfo{booktitle}{\emph{Proc. of the ACM SIGKDD Conference on Knowledge Discovery \& Data Mining}}. \bibinfo{pages}{3845--3853}.
\newblock


\bibitem[Zhang et~al\mbox{.}(2020)]%
        {tcad_zhang_2021}
\bibfield{author}{\bibinfo{person}{Xinqian Zhang}, \bibinfo{person}{Ming Hu}, \bibinfo{person}{Jun Xia}, \bibinfo{person}{Tongquan Wei}, \bibinfo{person}{Mingsong Chen}, {and} \bibinfo{person}{Shiyan Hu}.} \bibinfo{year}{2020}\natexlab{}.
\newblock \showarticletitle{Efficient federated learning for cloud-based AIoT applications}.
\newblock \bibinfo{journal}{\emph{IEEE Transactions on Computer-Aided Design of Integrated Circuits and Systems}} \bibinfo{volume}{40}, \bibinfo{number}{11} (\bibinfo{year}{2020}), \bibinfo{pages}{2211--2223}.
\newblock


\bibitem[Zhu et~al\mbox{.}(2021)]%
        {icml_zhuangdi_2021}
\bibfield{author}{\bibinfo{person}{Zhuangdi Zhu}, \bibinfo{person}{Junyuan Hong}, {and} \bibinfo{person}{Jiayu Zhou}.} \bibinfo{year}{2021}\natexlab{}.
\newblock \showarticletitle{Data-free knowledge distillation for heterogeneous federated learning}. In \bibinfo{booktitle}{\emph{Proc. of International Conference on Machine Learning}}. \bibinfo{pages}{12878--12889}.
\newblock


\end{thebibliography}

\appendix
\section{Proof of  FedMR Convergence}\label{sec:proof}

\subsection{Notations}
In our FedMR approach, the global model is aggregated from all the recombined models and all the models have the same weight. Let $t$ exhibit the $t^{th}$ SGD iteration on the local device, $v$ is the intermediate variable that represents the result of SGD update after exactly one iteration.
The update of FedMR is as follows:
\begin{equation}
v_{t+1}^k = w_t^k - \eta_t \nabla f_k (w_t^k, \xi_t^k), 
\end{equation}
\begin{equation}\label{eq:iter}
w_{t+1}^k=\left\{
\begin{array}{rl}
v_{t+1}^k, &if \quad E \nmid t + 1 \\
RM(v_{t+1}^{k}), &if \quad E \mid t + 1\\
\end{array}
,
\right.  
\end{equation}
where $w_{t}^k$ represents the model of the $k^{th}$ client in the $t^{th}$ iteration. 
$w_{t+1}$ denotes the global model of the ${(t+1)}^{th}$ iteration. 
$RM(v_{t+1}^{k})$ denotes the recombined model.

Since FedMR recombines all the local models in each round and the recombination only shuffles layers of models, the parameters of recombined models are all from the models before recombination, and no parameters are discarded. Therefore, when $\quad E \mid t + 1$, we can obtain the following invariants:
\begin{equation}\label{eq:mr_v1}
\sum_{k=1}^K v_{t+1}^k= \sum_{k=1}^K RM(v_{t+1}^k)= \sum_{k=1}^K w_{t+1}^k ,
\end{equation}
\begin{equation}\label{eq:mr_v2}
\sum_{k=1}^K||v_{t+1}^k-x||^2=\sum_{k=1}^K||w_{t+1}^k-x||^2 ,
\end{equation}
where $w_{t}^k$ is the $k^{th}$ recombined model in $(t-1)^{th}$ iteration, which is as the local model to be dispatched to $k^{th}$ client in $t^{th}$ iteration, $x$ can any vector with the same size as $v_{t}^k$.
Similar to \cite{sgd}, we define two variables $\overline{v}_t$ and $\overline{w}_t$:
\begin{equation}
\begin{split}
    \overline{v}_t = \frac{1}{K} \sum_{k=1}^{K} v_t^k, 
    \overline{w}_t = \frac{1}{K} \sum_{k=1}^{k} w_t^k.
\end{split}
\end{equation}

Inspired by \cite{convergence}, we make the following definition:
\begin{equation}
\begin{split}
    g^k_t =\nabla f_k (w_t^k; \xi_t^k).
\end{split}
\end{equation}

\subsection{Proof of Lemma \ref{key_lamma1}}
\begin{proof} 
Assume $v_{t}^k$ has $n$ layers, we have $v_{t}^k=L_1\oplus L_2 \oplus ... \oplus L_n$.
Let $L_i = [p^{v_t^k}_{(i,0)}, p^{v_t^k}_{(i,1)}, ... , p^{v_t^k}_{(i,|L_i|)}]$, where $p^{v_t^k}_{(i,j)}$ denotes the $j^{th}$ parameter of the layer $L_i$ in the model $v_{t}^k$.
We have
\begin{equation}
\begin{split}
    \sum_{k=1}^K||v_{t}^k-x||^2 = \sum_{k=1}^K\sum_{i=1}^n\sum_{j=1}^{|L_i|}||p^{v_t^k}_{(i,j)} - p^{x}_{(i,j)}||^2
\end{split}
\end{equation}
\begin{equation}
\begin{split}
    \sum_{k=1}^K||w_{t}^k-x||^2 = \sum_{k=1}^K\sum_{i=1}^n\sum_{j=1}^{|L_i|}||p^{w_t^k}_{(i,j)} - p^{x}_{(i,j)}||^2
\end{split}
\end{equation}

Since model recombination only shuffles layers of models, the parameters of recombined models are all from the models before recombination and no parameters are discarded. We have
\begin{equation}\label{eq:q2_0}
\begin{split}
    \forall_{i \in [1,n], j \in [1, |L_i|]} \sum_{k=1}^K p^{v_t^k}_{(i,j)} = \sum_{k=1}^K p^{w_t^k}_{(i,j)}
\end{split}
\end{equation}
\begin{equation}\label{eq:q2_1}
\begin{split}
    \forall_{k\in [1,K],i \in [1,n], j \in [1, |L_i|]} \exists_{q\in [1,K]}  \{{p^{v_t^k}_{(i,j)} = p^{w_t^{q}}_{(i,j)}}\}
\end{split}
\end{equation}
\begin{equation}\label{eq:q2_2}
\begin{split}
    \forall_{k\in [1,K],i \in [1,n], j \in [1, |L_i|]} \exists_{q\in [1,K]}  \{{p^{w_t^k}_{(i,j)} = p^{v_t^{q}}_{(i,j)}}\}
\end{split}
\end{equation}

According to Equations \ref{eq:q2_0}-\ref{eq:q2_2}, we have

\begin{equation}
\begin{split}
\sum_{k=1}^K||v_{t}^k-x||^2 & = \sum_{k=1}^K\sum_{i=1}^n\sum_{j=1}^{|L_i|}||p^{v_t^k}_{(i,j)} - p^{x}_{(i,j)}||^2 \\
& = \sum_{k=1}^K||w_{t}^k-x||^2\\
\end{split}
\end{equation}
\end{proof}

\subsection{Key Lemmas}
To facilitate the proof of our Theorem~\ref{thm1}, inspired by~\cite{convergence}, we can present the following two lemmas. 
Note that the following proofs are general proofs for all the multi-model-based FL approaches that satisfy Lemma~\ref{key_lamma1}.

\begin{lemma} \label{lemma1} (Results of one  step SGD). If $\eta_t\leq \frac{1}{4L}$, we have
\begin{equation}
\begin{split}
    \mathbb{E}||\overline{v}_{t+1} - w^\star||^2 \leq &\frac{1}{K}\sum_{k=1}^K(1-\mu\eta_t)||v^k_t -w^\star||^2\\
    & + \frac{1}{K}\sum_{k=1}^K||w^k_t - w^k_{t_0}||^2 + 10\eta_t^2 L\Gamma
\end{split}
. \nonumber
\end{equation}
\end{lemma}

\begin{proof} 
According to Lemma~\ref{key_lamma1} (i.e., Equation \ref{eq:mr_v1} and Equation \ref{eq:mr_v2}), we have
\begin{equation}
\begin{split}
    ||\overline{v}_{t+1} - w^\star||^2
    \leq  &\frac{1}{K}\sum_{k=1}^K||v^k_{t+1} - w^\star||^2\\
      = & \frac{1}{K}\sum_{k=1}^K(||v^k_t -w^\star||^2 -2\eta_t \langle w^k_t-w^\star, g^k_t \rangle\\
     & + \eta_t^2||g^k_t||^2)
\end{split}
\end{equation}
Let $B_1 =  -2\eta_t\langle w^k_t-w^\star, g^k_t \rangle$ and $B_2=\eta_t^2\sum_{k=1}^K||g^k_t||^2$.
According to Assumption \ref{asm2}, we have
\begin{equation}\label{eq_b1}
B_1 \leq -2\eta_{t}(f_k (w^k_t)-f_k (w^\star))-\mu\eta_t ||w^k_t-w^\star||^2
\end{equation}
According to Assumption \ref{asm1}, we have
\begin{equation}\label{eq_b2}
B_2 \leq 2\eta_t^2 L (f_k(w_t^k)-f_k^\star)
\end{equation}

According to Equation \ref{eq_b1} and \ref{eq_b2}, we have
\begin{equation}
\begin{split}
    ||\overline{v}_{t+1} - w^\star||^2 
    \leq &\frac{1}{K}\sum_{k=1}^K[(1-\mu\eta_t)||v^k_t -w^\star||^2 \\
    &-2\eta_{t}(f_k (w^k_t)-f_k (w^\star)) +2\eta_t^2 L(f_k(w_t^k)-f_k^\star)]
\end{split}
\end{equation}
    Let $C = \frac{1}{K}\sum_{k=1}^K[-2\eta_{t}(f_k (w^k_t)-f_k (w^\star)) + 2\eta_t^2 L (f_k(w_t^k)-f_k^\star)]$. We have
\begin{equation}
\begin{split}
    C &= \frac{-2\eta_{t}}{K}\sum_{k=1}^K(f_k (w^k_t)-f_k (w^\star)) + \frac{2\eta_t^2 L}{K}\sum_{k=1}^K(f_k(w_t^k)-f_k^\star)\\
    & = -\frac{2\eta_t(1-\eta_t L)}{K}\sum_{k=1}^K(f_k(w^k_t)-f^\star) + \frac{2\eta_t^2 L}{K}\sum_{k=1}^K(f^\star-f^\star_k)
\end{split}
\end{equation}
Let $\Gamma=f^\star-\frac{1}{K}\sum_{k=1}^K f^\star_k$ and $\phi=2\eta_t(1-L\eta_t)$. We have
\begin{equation}
    C = -\frac{\phi}{K}\sum_{k=1}^K(f_k(w^k_t)-f^\star) + 2\eta_t^2 L\Gamma
\end{equation}
Let $D=-\frac{1}{K}\sum_{k=1}^K(f_k(w^k_t)-f^\star)$, $E \mid  t_0$ and $t-t_0\leq E$. We have
\begin{equation}
\begin{split}
    D = -\frac{1}{K}\sum_{k=1}^K(f_k(w^k_t) - f_k(w^k_{t_0}) + f_k(w^k_{t_0}) -f^\star)
\end{split}
\end{equation}
By Cauchy–Schwarz inequality, we have
\begin{equation}\label{eq_D}
\begin{split}
    D \leq &\frac{1}{2K}\sum_{k=1}^K(\eta_t ||\nabla f_k(w^k_{t_0})||^2 + \frac{1}{\eta_t}||w^k_t - w^k_{t_0}||^2)\\
    &- \frac{1}{K}\sum_{k=1}^K(f_k(w^k_{t_0}) -f^\star)\\
     \leq &\frac{1}{2K}\sum_{k=1}^K[2\eta_t L(f_k(w^k_{t_0})-f_k^\star) + \frac{1}{\eta_t}||w^k_t - w^k_{t_0}||^2]\\
     &- \frac{1}{K}\sum_{k=1}^K(f_k(w^k_{t_0}) -f^\star)
\end{split}
\end{equation}

Note that since $\eta\leq \frac{1}{4L}$, $\eta_t \leq \phi\leq 2\eta_t$ and $\eta_t L \leq \frac{1}{4}$.
According to Equation \ref{eq_D}, we have
\begin{equation}
\begin{split}
    C 
     \leq &\frac{\phi}{2\eta_t K}\sum_{k=1}^K||w^k_t - w^k_{t_0}||^2 + (\phi \eta_t L + 2\eta_t^2 L)\Gamma + \frac{\phi}{K}\sum_{k=1}^K(f^\star - f_k^\star)\\
    \leq &\frac{\phi}{2\eta_t K}\sum_{k=1}^K||w^k_t - w^k_{t_0}||^2 + (\phi \eta_t L + \phi + 2\eta_t^2 L)\Gamma\\
    \leq  &\frac{1}{K}\sum_{k=1}^K||w^k_t - w^k_{t_0}||^2 + 10\eta_t^2 L \Gamma
\end{split}
\end{equation}
\end{proof}

\begin{lemma} \label{lemma2}
According to Equation~\ref{eq:iter}, the model recombination occurs every $E$ iterations.
Assume that in each training round, $t_0$ is the first iteration and iteration $t - t_0 \leq E-1$.
 Given the constraint on learning rate from \cite{convergence}, we know that $\eta_t \leq \eta_{t_0} \leq 2 \eta_t$. It follows that
\begin{equation}
\begin{split}
 \frac{1}{K} \sum_{k=1}^{K} ||w_t^k - {w}^k_{t_0}||^2 \leq 4\eta_t^2 (E - 1)^2 G^2.
\end{split}
\nonumber
\end{equation}
\end{lemma}

\begin{proof}
\begin{equation}
\begin{split}
\frac{1}{K} \sum_{k=1}^{K}||w_t^k - {w}^k_{t_0}||^2 = &\frac{1}{K} \sum_{k=1}^{K}||\sum_{t=t_0}^{t_0 + E - 1} \eta_t \nabla f_{a_1}(w_t^{a_1};\xi_t^{a_1})||^2   \\
& \leq  (t - t_0) \sum_{t = t_0}^{t_0 + E - 1} \eta_t^2 G^2  \\  
& \leq  (E - 1) \sum_{t = t_0}^{t_0 + E - 1} \eta_t^2 G^2  \\
&\leq 4\eta_t^2 (E - 1)^2 G^2.
\end{split}
\nonumber
\end{equation}
\end{proof}

\subsection{Proof of Theorem \ref{thm1}}
Based on Lemmas \ref{lemma1} and \ref{lemma2}, we can prove Theorem \ref{thm1} using the proof framework of FedAvg~\cite{convergence}. 
Due to space limitations, please refer to the proof of FedAvg~\cite{convergence} for the details.

\section{Secure Model Recombination Mechanism}\label{sec:sec_mr}

To avoid the risk of privacy leakage caused by  exposing gradients or models to the cloud server, we propose a 
secure model recombination mechanism for FedMR, which allows the random exchange of 
model layers among clients before model training or upload. As shown in 
Figure \ref{fig:sec_mr}, within a 
round of the  secure model recombination, the update of each model (i.e., $m$)  consists of four stages:

\begin{figure}[htp] 
	\begin{center} 		\includegraphics[width=0.45\textwidth]{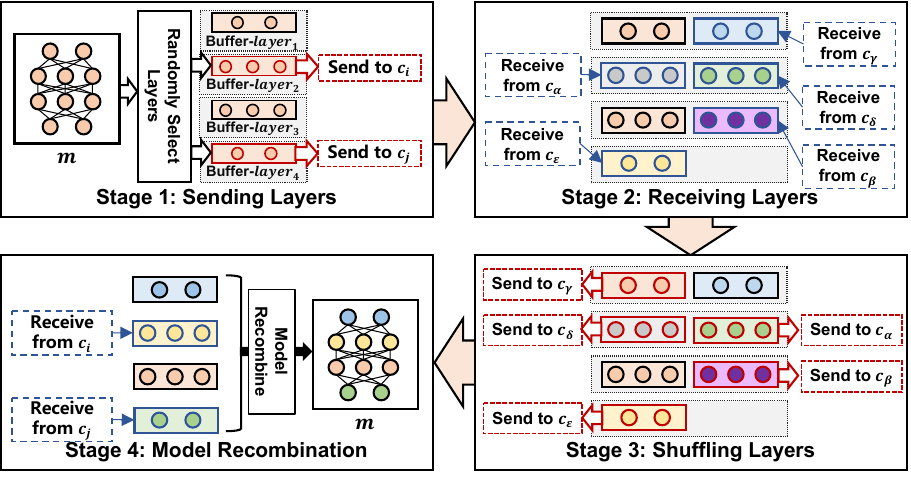}
\caption{Workflow of secure model recombination.}
		\label{fig:sec_mr} 
	\end{center}
\end{figure} 


\textbf{Stage 1:} Assume that the local model has $len$ layer. Each client maintains a  buffer for each layer.
Firstly, each client randomly selects
a part of its layers and sends them to other activated clients, while the 
remaining layers are saved in their corresponding buffers. 
Note that a selected
layer can only be sent to one client.
For example,   in Figure \ref{fig:sec_mr}, the client $m$ sends $layer_2$  and  $layer_4$ to 
$c_i$ and $c_j$, respectively.


\textbf{Stage 2:} 
Once receiving a layer from another client, the receiving 
client $m$ will add the layer to its corresponding 
buffer. For example, in Figure \ref{fig:sec_mr}, the client $m$ totally receives five layers.
Besides the retained two layers in stage 1, $m$
now has seven layers in total in its buffers.



\textbf{Stage 3:} 
For each layer buffer of $m$, if there contains 
one element received from  a client $c$ in stage 2, our  mechanism
will randomly select  one layer in the buffer and 
return it back to $c$. For example, 
in Figure \ref{fig:sec_mr},  $m$ randomly
returns a  layer in {\it Buffer-layer1}  back to a client $c_{\gamma}$.


\textbf{Stage 4:} 
Once receiving the returned layers from other clients, our mechanism
will recombine them with all the other layers in the buffers to form a new model. Note that the recombined model may significantly differ from the original model in Stage 1.


Note that each FL training
round can perform multiple times secure model recombination.  
Due to the randomness, it is hard for adversaries to 
figure out the sources of client model layers. 
In addition, the cloud server will broadcast a public key before the secure recombination to prevent privacy leakage. By using the public key to encrypt the model parameters of each layer, the other clients cannot directly obtain their received parameters.

\end{document}